%% file: neurips_2026.tex
\documentclass{article}

    \PassOptionsToPackage{numbers, compress}{natbib}
 \usepackage[main, final]{neurips_2026}
\usepackage[utf8]{inputenc} % allow utf-8 input
\usepackage[T1]{fontenc}    % use 8-bit T1 fonts
\usepackage{hyperref}       % hyperlinks
\usepackage{url}            % simple URL typesetting
\usepackage{booktabs}       % professional-quality tables
\usepackage{amsfonts}       % blackboard math symbols
\usepackage{nicefrac}       % compact symbols for 1/2, etc.
\usepackage{microtype}      % microtypography
\usepackage{xcolor}         % colors
\usepackage{graphicx}
\usepackage{wrapfig}
\usepackage{amsmath}
\usepackage{amssymb}
\usepackage{mathtools}
\usepackage{amsthm}
\usepackage{enumitem}

\usepackage[capitalize,noabbrev]{cleveref}

\crefname{section}{\S}{\S\S}
\crefname{figure}{Fig.}{Figs.}
\crefname{proposition}{Prop.}{Props.}
\Crefname{proposition}{Proposition}{Propositions}
\crefname{appendix}{Appx.}{Appxs.}
\crefname{theorem}{Thm.}{Thms.}
\crefname{definition}{Defn.}{Defns.}
\crefname{cor}{Corollary}{Corollaries}
\crefname{lem}{Lemma}{Lemmas}
\crefname{table}{Tab.}{Tabs.}
\crefname{assum}{Assum.}{Assums.}
\usepackage{aliascnt}

\theoremstyle{plain}
\newtheorem{theorem}{Theorem}[section]
\newaliascnt{proposition}{theorem}
\newtheorem{proposition}[proposition]{Proposition}
\aliascntresetthe{proposition}

\theoremstyle{definition}

\newtheorem{assumption}[theorem]{Assumption}
\theoremstyle{remark}

\usepackage[textsize=tiny]{todonotes}

\usepackage[most]{tcolorbox}
\newtcolorbox[auto counter]{counterexample}[1][]{%
  enhanced,
  breakable,
  colback=black!5,
  colframe=black!85,
  boxrule=1.25pt,
  arc=3pt,
  left=2mm,right=2mm,top=1mm,bottom=1mm,
  before skip=10pt, after skip=10pt,
  title={Counterexample~\thetcbcounter},
  colbacktitle=black!85,
  coltitle=white,
  fonttitle=\bfseries\small,
  attach boxed title to top left={yshift=-1.2mm, xshift=2mm},
  boxed title style={
    enhanced,
    arc=3pt,
    top=0.5mm, bottom=0.5mm, left=1mm, right=1mm,
    boxrule=0pt,
    interior engine=empty,
  },
  #1
}

\newtcolorbox[auto counter]{observation}[1][]{%
  enhanced,
  breakable,
  colback=blue!5,
  colframe=blue!85,
  boxrule=1pt,
  arc=2pt,
  left=1.5mm,right=1.5mm,top=0.3mm,bottom=0.3mm,
  before skip=4pt, after skip=4pt,
  title={Observation~\thetcbcounter},
  colbacktitle=black!85,
  coltitle=white,
  fonttitle=\bfseries\small,
  attach boxed title to top left={yshift=-0.9mm, xshift=1.5mm},
  boxed title style={
    enhanced,
    arc=2pt,
    top=0.2mm, bottom=0.2mm, left=0.8mm, right=0.8mm,
    boxrule=0pt,
    interior engine=empty,
  },
  #1
}

\usepackage{tikz}
\usetikzlibrary{arrows.meta, positioning, calc, backgrounds}

\definecolor{navy_blue}{HTML}{2454D4}
\definecolor{crimson_red}{HTML}{B71C1C}
\definecolor{slate_grey}{HTML}{37474F}
\definecolor{noise_amber}{HTML}{0097A7}
\definecolor{bg_contour}{HTML}{F0F2F5}

\title{Stabilizing Extrapolation in Looped Transformers via Learned Stochastic Stopping}

\author{%
  Hsun-Yu Kuo\textsuperscript{1,2}\thanks{Corresponding author: \texttt{hsun-yu.kuo@epfl.ch}} \quad
  El Mahdi Chayti\textsuperscript{1}\thanks{Equal mentorship.} \quad
  Patrik Reizinger\textsuperscript{2}\footnotemark[2] \quad
  Wieland Brendel\textsuperscript{2}\thanks{Equal supervision.} \quad
  Martin Jaggi\textsuperscript{1}\footnotemark[3]
  \\
  \textsuperscript{1}EPFL, Lausanne, Switzerland \\
  \textsuperscript{2}Max Planck Institute for Intelligent Systems \& ELLIS Institute, T\"ubingen, Germany
}

\begin{document}

\maketitle
\everypar{\looseness=-1}

\begin{abstract}
Looped Transformers --- which repeatedly apply a shared transformer block --- are an architecturally natural fit for variable-length algorithmic tasks. 
Although they can exhibit strong length generalization beyond the length of training sequences, this behavior is brittle, yielding high out-of-distribution (OOD) variance, even across well-performing in-distribution solutions. 
We trace this variance to the spurious correlation in simple algorithmic tasks between sequence length and number of loops.
Introducing stochasticity into the number of loops during training sharply reduces OOD variance and stabilizes predictions across inference-time loop counts.
To improve upon heuristic randomization schemes, we further analyze RL-Halting as a learned stochastic schedule and find that it generally improves the accuracy--stability trade-off.
Across binary addition, Dyck-1, Unique Set, and Copy, learned stochastic stopping often improves this trade-off but can also stabilize a suboptimal computation.
Our work suggests that “when to stop” should be treated as a training-time design choice, not merely an inference-time computation-allocation rule. The code for this work is available at \url{https://github.com/Hsun-Yu/looped_stopping}.
%We provide comparisons across binary addition, Dyck-1, Unique Set, and Copy, showing that learned stochastic stopping often improves this trade-off but can also stabilise a suboptimal computation. 
% Overall, our results suggest that ``when to stop'' should be treated as a training-time design choice in looped architectures, not merely as an inference-time compute-allocation rule.
\end{abstract}

\section{Introduction}
\label{sec:intro}

Transformers \cite{NIPS2017_3f5ee243} often achieve strong in-distribution performance, but can fail sharply when evaluation requires \emph{extrapolation}: inputs may be longer, harder, or more compositional than those seen during training~\cite{reizinger_position_2024, abbe_generalization_2024}. 
A common example is \emph{length generalisation}, where models are trained on short instances but tested on longer ones~\cite{zhou_what_2023,cai_extrapolation_2025,lee_self-improving_2025}. 
Algorithmic tasks such as multi-digit addition and planning-style tasks such as maze solving are representative cases: success requires executing a reusable procedure, not merely interpolating between familiar patterns.

A natural way to support such reusable computation is to add \emph{iteration via weight sharing}. 
Looped Transformers and related recurrent-depth architectures repeatedly apply a shared block to an internal state before producing an answer~\cite{dehghani_universal_2018,schwarzschild_can_2021,bansal_end--end_2022,giannou_looped_2023,fan_looped_2024,geiping_scaling_2025,zhu_scaling_2025}. 
This gives a simple computational intuition: if longer or harder instances require more repeated computation, then running the same block for more iterations may enable extrapolation. 
Indeed, prior work has shown that looped models can represent algorithmic computations and can sometimes discover extrapolating solutions~\cite{giannou_looped_2023,yang_looped_2023,fan_looped_2024}. 
However, expressivity alone does not answer the training question: even if an extrapolating computation exists in the architecture, gradient-based training may not reliably select it.

We show that this gap between \emph{architectural potential} and \emph{training outcome} is central in practice. 
On length-generalisation tasks, Looped Transformers can contain computations that solve inputs far beyond the training-length regime, yet this potential is highly unstable across runs. 
Runs with similar in-distribution accuracy can exhibit qualitatively different out-of-distribution (OOD) behaviour: some extrapolate, while others collapse soon after leaving the training regime. 
This suggests that looping does not simply add a uniformly helpful inductive bias. 
Instead, it introduces a training-time \emph{selection problem}: among the many computation trajectories available through repeated application of the same block, training must select one that is both stable and extrapolating.

We study this selection problem through the lens of \emph{stopping schedules}, which determine how many loop iterations are used during training. 
A deterministic schedule supervises each input at a single prescribed depth, potentially making the learned computation brittle to that choice. 
A stochastic schedule instead exposes the model to a local neighbourhood of loop depths, reducing dependence on any one stopping time. 
Empirically, we find that stochastic schedules substantially reduce run-to-run OOD variability and make predictions more stable across inference-time loop counts. 
However, stability alone is not sufficient: a schedule can stabilise a computation that still fails to extrapolate. 
We therefore also study RL-Halting, a learned stochastic stopping schedule related to adaptive-computation methods~\cite{banino_pondernet_2021,zhu_scaling_2025}. 
RL-Halting often improves the accuracy--stability trade-off across sequence tasks, but it can also stabilise a suboptimal computation.

Overall, our results suggest that ``when to stop'' is not merely an inference-time compute-allocation question. 
For looped architectures, it is also a training-time design choice that affects which computation trajectory is learned.
Our \textbf{contributions} are:
\begin{itemize}[leftmargin=*,noitemsep]
    \item \textbf{The extrapolation--stability gap.} 
    We show that deterministic Looped Transformers can contain extrapolating computations, but that these computations are not reliably selected across runs (\cref{sec:looping-enable}).

    \item \textbf{Stabilisation via stochastic schedules.} 
    We show that stochastic stopping schedules reduce run-to-run OOD variability and stabilise predictions across inference-time loop counts (\cref{sec:stochastic-schedules}).

    \item \textbf{Learning stochastic stopping schedules.} 
    We study RL-Halting as a learned stochastic schedule, finding that it often improves the accuracy--stability trade-off, while also showing that stable stopping does not always imply extrapolating stopping (\cref{sec:loss_rl_halting}).
\end{itemize}

\section{Background and Formulation}
\label{sec:background}

\subsection{Looped Transformers}

Looped Transformers implement iterative computation by repeatedly applying the same Transformer block across depth~\cite{dehghani_universal_2018,giannou_looped_2023,fan_looped_2024,geiping_scaling_2025,zhu_scaling_2025}. 
Rather than stacking K independent blocks, a looped model reuses one block across iterations, keeping parameter count fixed while varying computation depth.
% Rather than stacking $K$ independent blocks with separate parameters, a looped model reuses a single block for multiple iterations. 
% This weight tying keeps the parameter count fixed while allowing the amount of computation to vary with the number of loop steps.

In our experiments, we follow the looped architecture used by \citet{fan_looped_2024}. 
Let $x$ be an input sequence and let $E(x)\in\mathbb{R}^{m\times d}$ denote its token embeddings, where $m$ is the token-sequence length. 
The model maintains hidden states $H^{(t)}\in\mathbb{R}^{m\times d}$ for loop iteration $t$, initialised as $H^{(0)}=E(x)$. 
At each loop iteration, we apply the same Transformer block $\mathrm{Block}_{\theta}$:
% \patrik{use latex env block with begin-end, eg, aligned, equation or similar to number teh equations}
\begin{align}
H^{(t)}
=
\mathrm{Block}_{\theta}\!\left(H^{(t-1)} + H^{(0)}\right),
\qquad t=1,\ldots,K.
\end{align}
The addition of $H^{(0)}$ implements input injection \cite{fan_looped_2024,geiping_scaling_2025}: the original input representation is reintroduced at each iteration so that the loop can refine its state while retaining direct access to the input tokens.

After $K$ loop iterations, a readout map produces token-level logits from $H^{(K)}$, and training uses standard token-level cross-entropy on the supervised output positions. 
Supervision is applied only at the chosen loop depth, so training a Looped Transformer also requires specifying where the prediction is read out.
% Crucially, supervision is applied only at the chosen loop depth, rather than at every intermediate iteration. 
% Thus, unlike a standard fixed-depth Transformer, a Looped Transformer requires specifying not only the model parameters $\theta$, but also the loop depth at which the prediction is read out.

\subsection{Stopping Schedules as Training-Time Choices}

A stopping schedule determines how many loop iterations are used for an input. 
We write
$
K \sim \pi_{\mathrm{stop}}(K\mid x),
$
where $\pi_{\mathrm{stop}}$ is a stopping policy that may be deterministic, stochastic, or learned. 
We use $K$ for a generic loop count, and $\tau$ when emphasising a sampled stopping time, as in \cref{sec:loss_rl_halting}. 

For an example $(x,y)$, a stopping schedule induces the objective
\begin{align}
\mathcal{L}_{\pi}(\theta;x,y)
=
\mathbb{E}_{K\sim\pi_{\mathrm{stop}}(\cdot\mid x)}
\ell(\theta;x,y,K),
\end{align}
where $\ell(\theta;x,y,K)$ is the prediction loss after $K$ loop iterations. 
In practice, each optimisation step samples one loop depth per example. 
For deterministic schedules, this reduces to training at the prescribed depth. 
For stochastic schedules, the sampled-depth update is an unbiased estimate of the schedule-averaged objective.

This sampled-depth view is important because the stopping schedule affects training, not only inference. 
By choosing which loop depths receive supervision, the schedule can change which computation trajectory the model learns. 
Fixed schedules always train at a prescribed depth, length-matched schedules tie the depth to the input length, and adaptive-computation methods learn or infer when to stop~\cite{saunshi_reasoning_2024,fan_looped_2024,bansal_thinking_2021,banino_pondernet_2021,zhu_scaling_2025}. 
In the experiments, we compare deterministic schedules, hand-designed stochastic schedules, and a learned stochastic stopping rule.

\subsection{RL-Halting: Learned Sampled-Depth Stopping}
\label{sec:rl-halting-formulation}

We introduce \emph{RL-Halting} as a learned stochastic stopping rule for Looped Transformers. 
A lightweight stopping head parameterises $\pi_{\mathrm{stop},\phi}(\tau\mid x)$, a distribution over loop depths. 
During training, we sample one stopping time $\tau$, supervise the model only at that realised depth, and train the stopping rule using the negative sampled-depth loss as reward,
$
R(x,y,\tau)=-\ell(\theta;x,y,\tau).
$
The stopping parameters are updated with a score-function estimator, also known as REINFORCE~\citep{williams_simple_nodate,murphy_reinforcement_2025},
$
(R-b)\nabla_\phi \log \pi_{\mathrm{stop},\phi}(\tau\mid x),
$
with a moving-average baseline $b$ and an entropy regulariser for exploration. 
RL-Halting therefore replaces a hand-designed depth distribution with a loss-driven learned one while preserving the same one-sampled-depth protocol. 
Details are given in \cref{app:rl-halting}.
\section{Experimental Setup}
\label{sec:setup}

Unless otherwise stated, our diagnostic experiments use binary addition; cross-task results are reported in \cref{sec:loss_rl_halting}.

\subsection{Tasks and length-generalisation protocol}
\label{sec:tasks-protocol}

We evaluate length generalisation under a controlled train--test length shift. 
Let $n$ denote the problem length. 
For binary addition, $n$ is the number of digits in each operand; for the other sequence tasks, $n$ is the corresponding input length or problem size. 
All models are trained on lengths $n<20$, which we call the in-distribution regime (ID), and evaluated on lengths $n\geq 20$, which we call OOD. 
Unless otherwise stated, OOD summaries are computed over lengths $\{20,25,30,\ldots,60\}$.

Our main task is binary addition. 
Each example contains two $n$-digit binary numbers, and the target is their sum followed by an end token; for example, $1011 + 0110$ maps to $10001\,\texttt{<eos>}$. 
Binary addition admits a linear-time algorithm, but vanilla Transformers struggle to length-generalise on this task without explicit index hints~\cite{zhou_what_2023,fan_looped_2024}.

We also report results on Dyck-1, Unique Set, and Copy. 
Dyck-1 tests generalisation in a simple formal-language setting with nested structure. 
Unique Set maps a sequence to the unique symbols in order of first appearance, e.g. $[a,b,a,c,b]\mapsto [a,b,c,\texttt{<eos>}]$. 
Copy maps the input sequence to itself followed by an end token. 
For all tasks, a prediction is correct only if the complete output sequence is correct, including the end token.

\subsection{Models and stopping schedules}
\label{sec:models-schedules}

We compare fixed-depth Transformers with Looped Transformers. 
The fixed-depth baselines are a 3-layer Transformer, matching the parameter count of one looped block, and a 60-layer Transformer, serving as a higher-compute baseline. 
The Looped Transformer reuses a shared 3-layer block across loop iterations.

For Looped Transformers, we evaluate three families of stopping schedules.

\textbf{Deterministic schedules.}
We consider a fixed schedule, $K=20$, as in fixed-depth latent-reasoning settings~\citep{saunshi_reasoning_2024}, and a length-matched schedule, $K=n$, where $n$ is the input length~\citep{fan_looped_2024}.

\textbf{Hand-designed stochastic schedules.}
Starting from the length-matched schedule, we sample
$
\Delta \sim \mathrm{Unif}\{-w,\ldots,w\}
$
and train at
$
K=\mathrm{clip}(n+\Delta,1,T_{\max}).
$
Unless otherwise stated, stochastic schedules sample one loop depth per example per optimisation step.

\textbf{Learned stochastic schedules.}
We study RL-Halting, the learned sampled-depth stopping rule introduced in \cref{sec:rl-halting-formulation}. 
During training, RL-Halting samples one stopping depth per example and updates the stopping distribution using the negative realised task loss as reward. 
This gives a learned counterpart to the hand-designed stochastic schedules.

% Additional implementation details are given in \cref{app:implementation-details}.
Additional implementation details are given in
\cref{app:implementation-details}, and a computational cost
comparison across stopping schedules is provided in
\cref{app:compute-comparison}.

\subsection{Evaluation metrics}
\label{sec:metrics}

Our primary diagnostic metric is
\emph{oracle-over-iterations accuracy}.
For examples of length $n$, we define
\begin{align}
\mathrm{Acc}_{\mathrm{oracle}}(n)
=
\mathbb{E}_{(x,y)\sim\mathcal{D}_n}
\left[
\max_{K\in\mathcal{K}_{\mathrm{eval}}(n)}
\mathrm{acc}\!\left(\hat y_K(x),y\right)
\right],
\end{align}
where $\hat y_K(x)$ is the prediction after $K$ loop
iterations and
$\mathcal{K}_{\mathrm{eval}}(n)=\{1,\ldots,n+2\}$.
Accuracy requires exact correctness of the complete
output sequence, including the end token.
For fixed-depth Transformer baselines, which have no
loop depth to select, we evaluate the final prediction
at the model's fixed depth.

Oracle-over-iterations accuracy is not an inference
policy: it uses the ground truth to select the best
loop depth within the evaluated range.
We use it to diagnose whether a correct prediction
exists at some depth, separately from whether a
practical stopping rule selects it.

We therefore also evaluate \emph{policy accuracy},
which uses each method's prescribed inference-time
stopping rule without ground-truth depth selection.
We use $K=20$ for fixed-$20$ loop training and $K=n$
for both length-matched and randomised $K=n\pm5$
training.
For RL-Halting, we select the highest-probability depth:
\begin{align}
K_\pi(x)
=
\arg\max_{1\leq t\leq T_{\mathrm{test}}}
\pi_{\mathrm{stop},\phi}^{(T_{\mathrm{test}})}(t\mid x),
\qquad T_{\mathrm{test}}=30,
\end{align}
where the truncated stopping distribution is defined
in \cref{app:rl-halting}.
Thus, RL-Halting samples stopping depths during training
but uses maximum-probability stopping for policy
evaluation.

For summary tables, we report five metrics.
\emph{OOD} is the average oracle-over-iterations accuracy
over evaluated OOD lengths $\{20,25,30,\ldots,60\}$,
averaged across runs.
Policy OOD accuracy is likewise averaged over the evaluated
OOD lengths and then across available runs.
\emph{Front.@90} is the mean-accuracy frontier:
the largest evaluated OOD length at which the mean
oracle accuracy across runs reaches at least $90\%$.
\emph{Std.} is the standard deviation across runs of
per-run average oracle OOD accuracy.
\emph{Gap} is aggregate oracle OOD accuracy minus
aggregate policy OOD accuracy.
OOD is reported in percent; Std. and Gap are in
percentage points.
Full policy results and aggregation details are
provided in \cref{app:oracle-policy-gap}.

For length-wise figures, we additionally show the mean
and standard deviation across independent runs at each
input length.

\section{Experimental Results}
\label{sec:experiments}

\subsection{Deterministic Looping Reveals Extrapolation Potential but Is Unstable}
\label{sec:looping-enable}

\begin{figure*}[t]
  \centering
  \includegraphics[width=\textwidth]{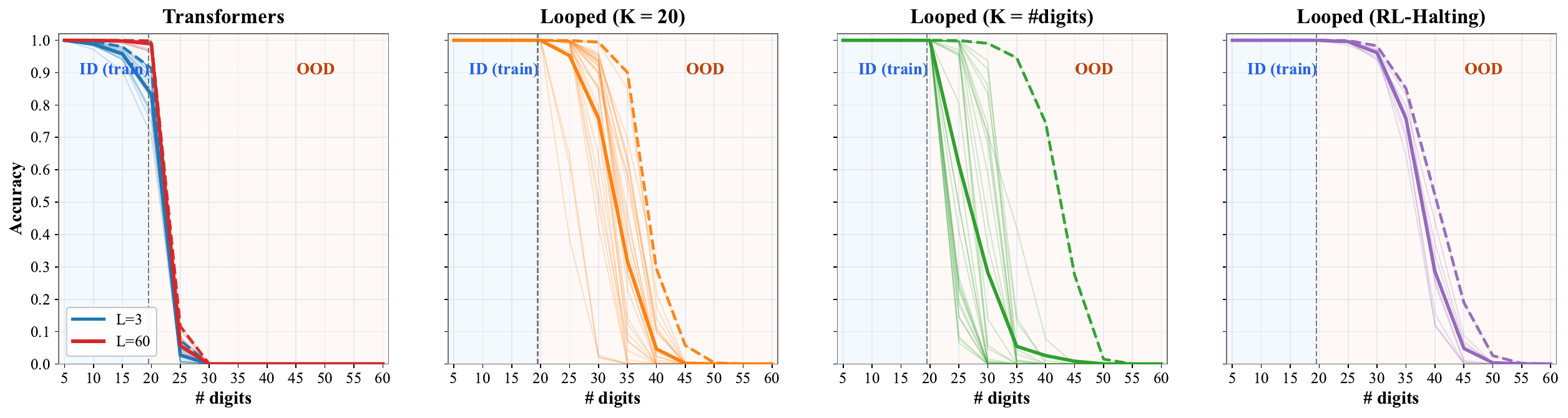}
% \vspace{-5pt}
\caption{
\textbf{Oracle-over-iterations accuracy on binary addition as a function of input length for standard Transformers and Looped Transformers with fixed $K=20$, length-matched $K=n$, and RL-Halting.} 
RL-Halting is shown for reference and discussed in \cref{sec:loss_rl_halting}. 
The shaded region marks the training-length regime (ID), and the right-hand side corresponds to longer OOD inputs. 
Thin solid lines denote individual runs, the bold solid line denotes the mean across runs, and the dashed line denotes the per-length maximum across runs.
}
\label{fig:accuracy-runs-comparison}
\vspace{8pt}
\end{figure*}

% We first ask whether deterministic looped models reliably learn computations that extrapolate beyond the training-length regime.
% This isolates the effect of looping itself before introducing stochastic or learned stopping schedules.
We first isolate whether deterministic looped models reliably learn computations that extrapolate beyond the training-length regime.
The answer is mixed: looped models can contain extrapolating computations, but these computations are not reliably selected across training runs.

\textbf{Looping reveals extrapolation potential, but not reliable OOD performance.}
\label{sec:potential-instability}

\cref{fig:accuracy-runs-comparison} compares fixed-depth Transformers with Looped Transformers on binary addition.
The fixed-depth Transformer baselines fit the ID regime but degrade sharply beyond the training-length boundary.
Looped Transformers, in contrast, show clear extrapolation potential: for both fixed $K=20$ and length-matched $K=n$, the best runs remain accurate on substantially longer inputs.

This potential is not reflected in reliable average performance.
The individual run curves in \cref{fig:accuracy-runs-comparison} show large OOD variation: some looped runs extrapolate far beyond the training lengths, while others collapse soon after entering the OOD regime.
This is also visible quantitatively in \cref{fig:ood-variance-comparison}: the OOD standard deviation is only $0.009$ and $0.005$ for the 3-layer and 60-layer Transformer baselines, but rises to $0.087$ for Looped Transformers with fixed $K=20$ and $0.122$ for Looped Transformers with $K=n$.
Controlled variants, where either the initialisation or the data order is fixed, show the same qualitative instability; see \cref{app:controlled-variance} for details.

\textbf{Larger loop budgets widen the gap between potential and reliability.}
\label{sec:loop-budget-instability}

\begin{wrapfigure}{r}{0.55\textwidth}
  \vspace{-11pt}
  \centering
  \includegraphics[width=\linewidth]{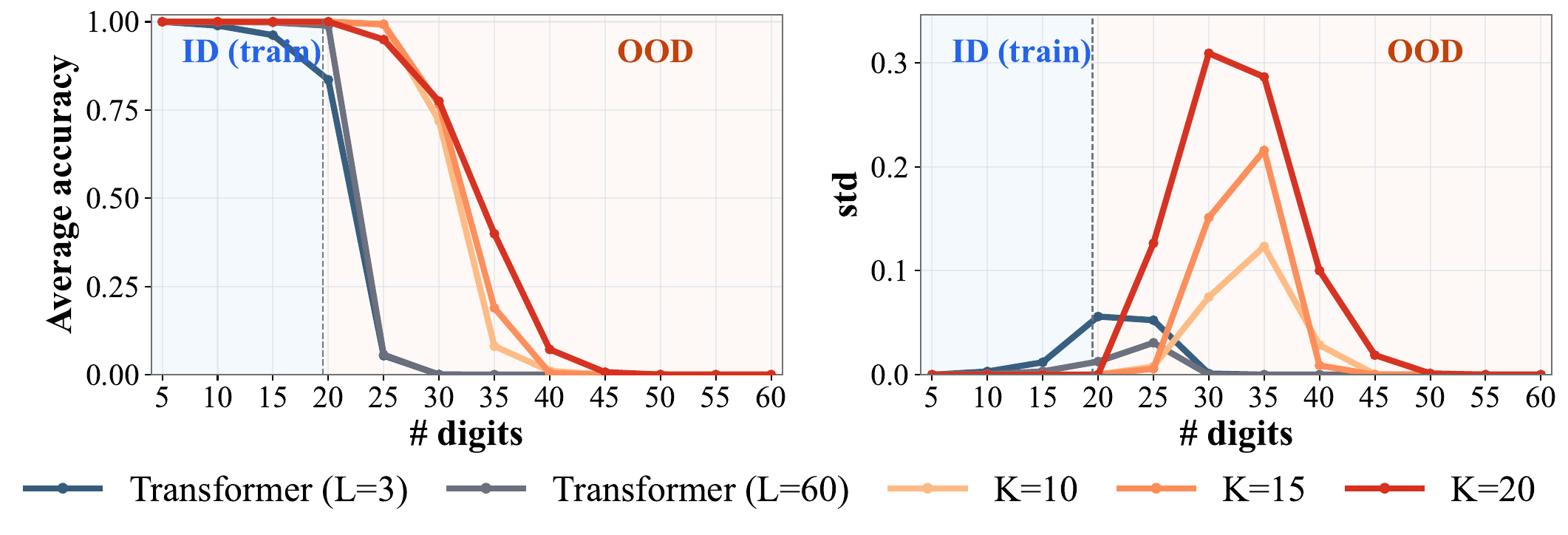}
  \caption{
\textbf{Mean performance and run-to-run variability on binary addition across input lengths. }
\textbf{Left}: mean oracle-over-iterations accuracy across runs. 
\textbf{Right}: standard deviation across runs at each digit length. 
We compare standard Transformer baselines with fixed-$K$ Looped Transformers for $K \in \{10,15,20\}$. 
The shaded region marks the training-length regime (ID); longer inputs are OOD.}
  \label{fig:accuracy-potential-by-length}
\vspace{-10pt}

\end{wrapfigure}

We next vary the fixed loop budget.
As shown in \cref{fig:accuracy-potential-by-length}, increasing $K$ shifts the mean accuracy curve toward longer inputs, suggesting that additional loop computation can extend the range over which some models extrapolate.
However, the same increase also raises run-to-run variability in the OOD regime, especially at intermediate lengths where some runs still solve the task while others have already collapsed.

Thus, larger loop budgets increase extrapolation potential, but also make OOD behaviour less reliable.
This supports the view that looping expands the computation trajectories available to training: some extrapolate, while others fit ID but fail OOD.
% This supports the view that looping expands the set of computation trajectories available to training: some trajectories extrapolate, while others fit the ID regime but fail OOD.

\begin{observation}
\paragraph{Looping creates a selection problem.}
Looping does not merely add computation; it expands the set of computation trajectories that training can select.
Some trajectories extrapolate, but others fit the ID regime while failing OOD.
Thus, the central difficulty is not only whether a looped architecture \emph{can} represent an extrapolating computation, but whether training reliably selects one.
\end{observation}

\subsection{Stochastic Schedules Stabilise Looped Models}
\label{sec:stochastic-schedules}

We now test whether this selection problem can be regularised by randomising the training loop depth. 
Starting from the length-matched schedule $K=n$, we train stochastic variants that sample from a local window around $n$. 
We evaluate whether these schedules reduce run-to-run OOD variability and make each model less sensitive to inference-time loop depth. 
A toy explanation is given in \cref{sec:toy-analysis}.

\textbf{Randomised schedules reduce run-to-run variability.}
\label{sec:randomness-less-fragile}

\begin{wrapfigure}{r}{0.55\textwidth}
\vspace{-5pt}
  \centering
  \includegraphics[width=\linewidth]{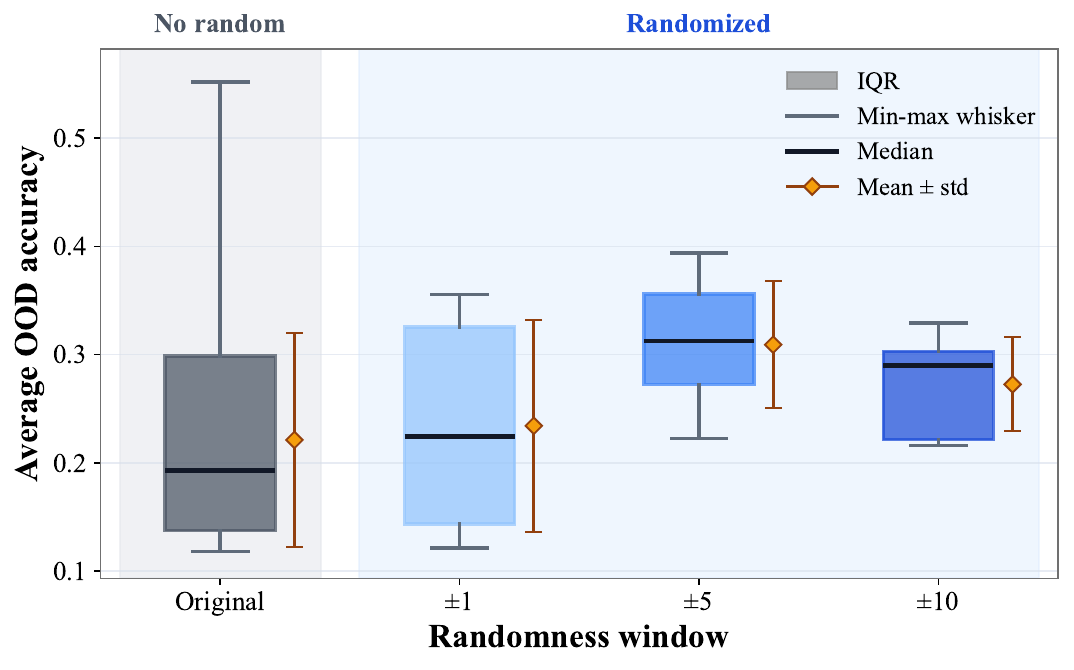}
  \caption{
\textbf{OOD oracle-over-iterations accuracy for stochastic variants of the length-matched schedule on binary addition.}
``Original'' denotes deterministic $K=n$; stochastic variants sample $K$ from a local window around $n$.
Boxes show the interquartile range across runs, centre lines show the median, whiskers show the min--max range, and orange diamonds with error bars show the mean $\pm$ one standard deviation.
}
%   \caption{
% \textbf{OOD oracle-over-iterations accuracy for stochastic variants of the length-matched schedule on binary addition.} 
% ``Original'' denotes the deterministic schedule $K=n$; stochastic variants sample $K$ from a local window around $n$. 
% Dots show the mean across runs, thick vertical bars show $\pm$ one standard deviation, and light horizontal bars show the min--max range.
% \patrik{a box plot basically does the same and it's standard - you might consider replacing it}}
\label{fig:random-window-ood-summary}
\vspace{-10pt}
\end{wrapfigure}

\cref{fig:random-window-ood-summary} shows that stochastic schedules substantially reduce the spread of OOD outcomes across runs. 

The deterministic length-matched schedule has a wide interquartile range and a large min--max range, indicating that different runs select very different OOD behaviours.
By contrast, the $\pm5$ and $\pm10$ windows produce more concentrated distributions and smaller standard deviations.
A fixed-budget ablation in \cref{app:fixed-k-random-ablation} shows a similar stabilising effect when randomising around $K=20$.
% The deterministic length-matched schedule has a wide min--max range, indicating that different runs select very different OOD behaviours. 
% By contrast, the $\pm5$ and $\pm10$ windows produce much tighter ranges and smaller standard deviations. A fixed-budget ablation in \cref{app:fixed-k-random-ablation} shows a similar stabilising effect when randomising around $K=20$.

The effect is primarily one of stabilisation, not a monotonic improvement in accuracy. 
Central OOD performance improves from the deterministic schedule to the $\pm5$ window, but decreases again for the $\pm10$ window.
% The mean OOD accuracy increases from the deterministic schedule to the $\pm5$ window, but decreases again for the $\pm10$ window. 
Thus, randomisation can reduce OOD variance, but too much randomness need not further improve extrapolation.

\textbf{Randomised schedules stabilise predictions across loop depths.}
\label{sec:loop-trajectory-stability}

\begin{figure*}[t]
  \centering
  \includegraphics[width=\textwidth]{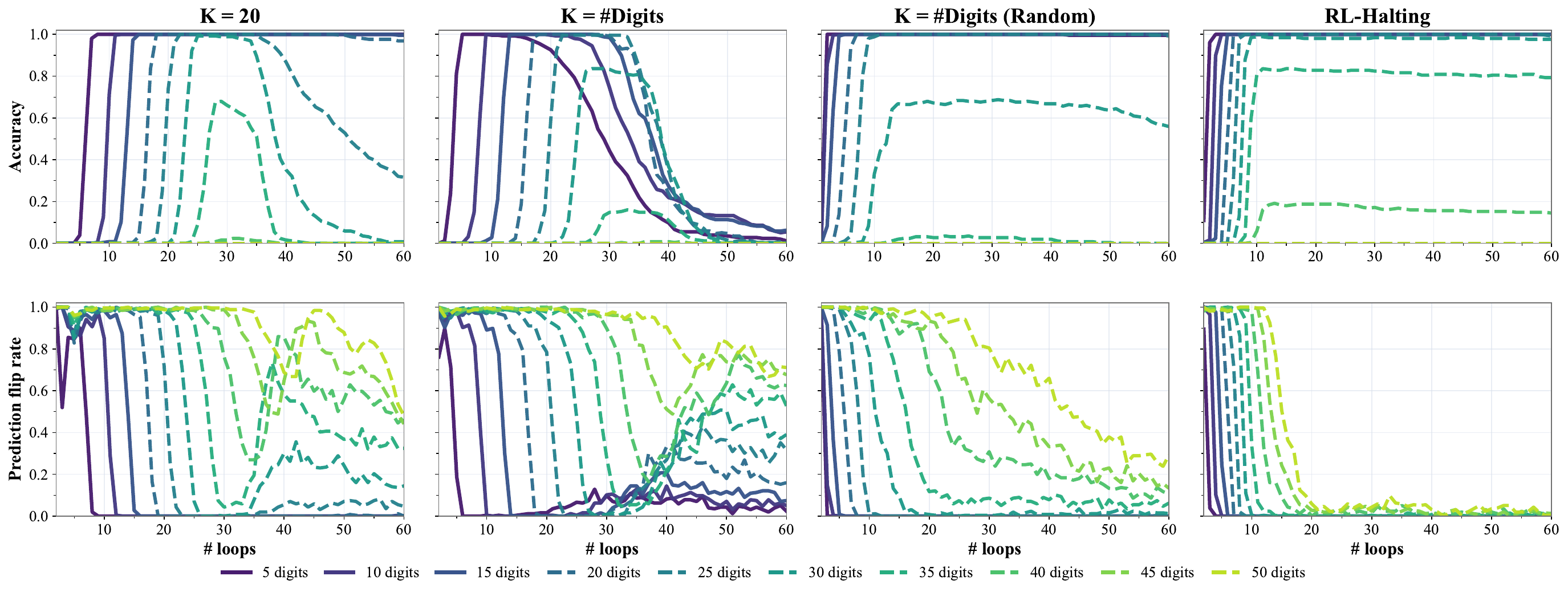}
  \caption{
\textbf{Accuracy and prediction dynamics across inference-time loop counts on binary addition.}
Columns compare models trained with fixed $K=20$, length-matched $K=\#\mathrm{Digits}$, randomised length-matched $K=\#\mathrm{Digits}$ (Random), and RL-Halting. 
Here $K=\#\mathrm{Digits}$ (Random) denotes training with $K=\mathrm{clip}(n+\Delta,1,T_{\max})$, where $\Delta\sim\mathrm{Unif}\{-5,\ldots,5\}$.
\textbf{Top row:} accuracy as a function of inference-time loop count $K$ for different input lengths. 
\textbf{Bottom row:} prediction flip rate between adjacent loop counts. 
Lower flip rates indicate more stable loop trajectories. 
RL-Halting is shown for reference and discussed in \cref{sec:loss_rl_halting}.
}
  \label{fig:loop-trajectory-stability}
\end{figure*}

We next test whether stochastic schedules also stabilise the loop trajectory within a trained model. 
For each model, we sweep the inference-time loop count $K$ and measure accuracy at each depth. 
We also measure the prediction flip rate,
$
\Pr_x[\hat y_{K+1}(x)\neq \hat y_K(x)],
$
which records how often the predicted sequence changes between adjacent loop depths.

\cref{fig:loop-trajectory-stability} shows that deterministic schedules are sensitive to the inference-time stopping depth. 
They often achieve high accuracy only over a narrow range of loop counts, and their predictions can continue changing across many iterations. 
Randomised training reduces this sensitivity: the randomised length-matched schedule gives a broader high-accuracy region across $K$ and a faster decrease in prediction flip rate. 
This matters in practice because the best loop depth is unknown at test time: in \cref{tab:ood_generalization,app:oracle-policy-gap}, we show that more stable loop trajectories reduce the gap between oracle-over-iterations accuracy and the accuracy obtained by the method's selected stopping rule.

\begin{observation}
\paragraph{Stochastic schedules stabilise looped models along two axes.}
Randomising the training loop depth reduces OOD variability across independent training runs and makes predictions less sensitive to the inference-time loop count within a run. 
Thus, stochastic schedules regularise both which computation trajectory training selects and how sharply performance depends on the stopping depth. 
However, this stabilisation does not by itself guarantee that the selected trajectory extrapolates best.
\end{observation}

% Together, \cref{fig:random-window-ood-summary,fig:loop-trajectory-stability} show that stochastic schedules stabilise looped models in two complementary ways: across independent training runs, and across inference-time loop depths within a run.

% \begin{observation}
% \paragraph{Stochastic schedules stabilise selection, but do not guarantee extrapolation.}
% Randomising the training loop depth reduces dependence on a single prescribed stopping time. 
% This makes looped models less fragile across both training runs and inference-time depths. 
% However, randomisation only stabilises the computation selected by training; it does not ensure that the selected computation is the most extrapolating one. 
% This motivates learning the stopping distribution rather than fixing it by hand.
% \end{observation}

\subsection{Learned Stopping Can Improve the Accuracy--Stability Trade-off}
\label{sec:loss_rl_halting}

The stochastic schedules above show that sampling loop depths can stabilise looped models, but their distributions are hand-designed and depend on the chosen window size. 
We therefore ask whether the stopping distribution itself can be learned while preserving the sampled-depth training protocol. 
We evaluate RL-Halting, introduced in \cref{sec:rl-halting-formulation}, as an adaptive intervention for studying the accuracy--stability trade-off in length generalisation.

% RL-Halting differs from the hand-designed stochastic schedules in one key respect: instead of sampling from a fixed window around the input length, it learns a loss-driven distribution over stopping depths. 
% At the same time, it keeps the comparison aligned with \cref{sec:stochastic-schedules}, because each update still supervises the Looped Transformer at a single sampled depth. 
% Implementation details are given in \cref{app:rl-halting}. 
% In our training setup, RL-Halting has comparable wall-clock
% cost to the other schedules and nearly identical recurrent
% forward computation to the randomised $K=n\pm5$ schedule
% (\cref{app:compute-comparison}).
% On binary addition, we further inspect the learned stopping distribution and its conditional entropy in \cref{app:learned-stopping-schedule-addition}, suggesting that RL-Halting learns a length-dependent but non-degenerate stochastic schedule.

RL-Halting differs from the hand-designed stochastic schedules
by learning a loss-driven distribution over stopping depths
instead of sampling from a fixed window around the input length.
It preserves the single-sampled-depth supervision protocol used
in \cref{sec:stochastic-schedules}.
In our training setup, its wall-clock cost is comparable to the
other schedules, and its recurrent forward computation is nearly
identical to that of the randomised $K=n\pm5$ schedule;
implementation details and the computational cost comparison
are provided in \cref{app:rl-halting} and \cref{app:compute-comparison},
respectively.
On binary addition, inspection of the learned stopping
distribution and its conditional entropy shows that the schedule
adapts to input length while remaining non-degenerate
(\cref{app:learned-stopping-schedule-addition}).

% The stochastic schedules above show that sampling loop depths can stabilise looped models, but their distributions are hand-designed and depend on the chosen window size. 
% We therefore ask whether the stopping distribution itself can be learned. 
% Our goal is not to introduce a new adaptive-computation method, but to use learned stopping as an adaptive intervention for studying the accuracy--stability trade-off in length generalisation.

% We use \emph{RL-Halting} as an adaptive stochastic stopping rule. 
% A small stopping head defines a distribution over loop depths, samples one depth for each training example, and supervises the Looped Transformer only at that depth. 
% This matches the hand-designed stochastic schedules in \cref{sec:stochastic-schedules} in one important respect: each update is based on the loss at a single sampled loop depth. 
% Unlike those schedules, however, RL-Halting learns the sampling distribution from the realised loss rather than fixing it by hand. 
% Implementation details, including the stopping distribution and policy-gradient update, are given in \cref{app:rl-halting}. 
% On binary addition, we further inspect the learned stopping distribution and its conditional entropy in \cref{app:learned-stopping-schedule-addition}, suggesting that RL-Halting learns a length-dependent but non-degenerate stochastic schedule.
% Implementation details, including the stopping distribution and policy-gradient update, are given in \cref{app:rl-halting}.

\textbf{Learned schedules often improve the trade-off.}

% \begin{wraptable}{r}{0.5\textwidth}
% \vspace{-18pt}
% \centering
% \scriptsize
% \setlength{\tabcolsep}{2.2pt}
% \renewcommand{\arraystretch}{0.88}
% \caption{
% \textbf{OOD performance and stability across tasks. }
% OOD is accuracy over lengths 20--60, Front.@90 is the mean-accuracy frontier at $90\%$, and Std. is the standard deviation across runs.
% }
% \label{tab:ood_generalization}
% \begin{tabular}{llccc}
% \toprule
% Task & Method & OOD $\uparrow$ & Front.@90 $\uparrow$ & Std. $\downarrow$ \\
% \midrule
% Addition 
% & Fix $K=20$ & 34.2 & 25 & 7.4 \\
% & $K=n$ & 22.1 & 20 & 9.9 \\
% & $K=n,\pm5$ & 30.9 & 25 & 5.9 \\
% & RL-Halt & \textbf{45.0} & \textbf{30} & \textbf{2.7} \\
% \midrule
% Dyck-1 
% & Fix $K=20$ & 84.6 & 30 & 23.8 \\
% & $K=n$ & 66.6 & 25 & 33.3 \\
% & $K=n,\pm5$ & 88.2 & 35 & 18.7 \\
% & RL-Halt & \textbf{97.5} & \textbf{60} & \textbf{3.6} \\
% \midrule
% Unique 
% & Fix $K=20$ & 66.1 & 20 & 34.0 \\
% & $K=n$ & 78.5 & 35 & 17.5 \\
% & $K=n,\pm5$ & 73.2 & \textbf{40} & 9.3 \\
% & RL-Halt & \textbf{82.8} & \textbf{40} & \textbf{3.3} \\
% \midrule
% Copy 
% & Fix $K=20$ & 62.3 & 25 & 27.7 \\
% & $K=n$ & \textbf{68.7} & \textbf{35} & 13.2 \\
% & $K=n,\pm5$ & 55.9 & 30 & 8.5 \\
% & RL-Halt & 43.2 & 25 & \textbf{2.7} \\
% \bottomrule
% \end{tabular}
% \vspace{-8pt}
% \end{wraptable}

\begin{wraptable}{r}{0.55\textwidth}
\vspace{-18pt}
\centering
\scriptsize
\setlength{\tabcolsep}{1.8pt}
\renewcommand{\arraystretch}{0.88}
\caption{
\textbf{OOD performance and stability across tasks.}
OOD is mean oracle accuracy (\%) over lengths 20--60.
Front.@90 is the largest evaluated length with mean oracle
accuracy $\geq90\%$, and Std. is the across-run standard
deviation of OOD accuracy.
Gap is oracle minus policy OOD accuracy.
Std. and Gap are in percentage points.
}
\label{tab:ood_generalization}
\begin{tabular}{@{}llrrrr@{}}
\toprule
Task & Method & OOD $\uparrow$
& \shortstack{Front.@90 $\uparrow$}
& Std. $\downarrow$ & Gap $\downarrow$ \\
\midrule
Addition
& $K=20$     & 34.2 & 25 & 7.4 & 16.4 \\
& $K=n$      & 22.1 & 20 & 9.9 & 7.8 \\
& $K=n\pm5$  & 30.9 & 25 & 5.9 & \textbf{1.3} \\
& RL-Halt    & \textbf{45.0} & \textbf{30}
             & \textbf{2.7} & 5.8 \\
\midrule
Dyck-1
& $K=20$     & 84.6 & 30 & 23.8 & 32.8 \\
& $K=n$      & 66.6 & 25 & 33.3 & 32.4 \\
& $K=n\pm5$  & 88.2 & 35 & 18.7 & \textbf{2.8} \\
& RL-Halt    & \textbf{97.5} & \textbf{60}
             & \textbf{3.6} & 4.3 \\
\midrule
Unique
& $K=20$     & 66.1 & 20 & 34.0 & 41.7 \\
& $K=n$      & 78.5 & 35 & 17.5 & 14.0 \\
& $K=n\pm5$  & 73.2 & \textbf{40}
             & 9.3 & \textbf{4.4} \\
& RL-Halt    & \textbf{82.8} & \textbf{40}
             & \textbf{3.3} & 5.3 \\
\midrule
Copy
& $K=20$     & 62.3 & 25 & 27.7 & 29.3 \\
& $K=n$      & \textbf{68.7} & \textbf{35}
             & 13.2 & 24.9 \\
& $K=n\pm5$  & 55.9 & 30 & 8.5 & \textbf{1.5} \\
& RL-Halt    & 43.2 & 25 & \textbf{2.7} & 1.9 \\
\bottomrule
\end{tabular}
\vspace{-8pt}
\end{wraptable}

% \cref{tab:ood_generalization} summarises OOD accuracy, extrapolation frontier, and run-to-run stability across four tasks. 
% On Addition, Dyck-1, and Unique, RL-Halting gives the best OOD accuracy while also giving the lowest OOD standard deviation. 
% This suggests that the stabilising effect of sampled-depth training can be made adaptive: rather than fixing a depth window in advance, the model can learn where to allocate stopping probability.
% \cref{tab:ood_generalization} summarises oracle OOD accuracy,
% extrapolation frontier, run-to-run variability, and the
% oracle--policy gap across four tasks.
% On Addition, Dyck-1, and Unique Set, RL-Halting achieves
% the highest oracle accuracy and the lowest OOD standard
% deviation.
% Its policy accuracy is also highest on these tasks:
% $39.2\%$, $93.2\%$, and $77.5\%$, respectively, compared
% with $29.6\%$, $85.4\%$, and $68.8\%$ for $K=n\pm5$.
% Thus, its accuracy gains persist without oracle depth
% selection, although $K=n\pm5$ has the smallest gap
% on every task.

\cref{tab:ood_generalization} summarises oracle OOD accuracy,
extrapolation frontier, run-to-run variability, and the
oracle--policy gap across four tasks.
On Addition, Dyck-1, and Unique Set, RL-Halting achieves
the highest oracle accuracy and the lowest OOD standard
deviation.
Its policy accuracy is also highest on these tasks:
$39.2\%$, $93.2\%$, and $77.5\%$, respectively, compared
with $29.6\%$, $85.4\%$, and $68.8\%$ for $K=n\pm5$.
Thus, its accuracy gains persist without oracle depth
selection.

Both stochastic schedules substantially reduce the
oracle--policy gap relative to the deterministic schedules
on every task.
Although $K=n\pm5$ achieves the smallest gap, RL-Halting
offers higher oracle and policy accuracy together with
lower run-to-run variability on three tasks.
Learning the stopping distribution therefore improves
the accuracy--stability trade-off on these tasks,
without necessarily minimising the oracle--policy gap.

% \textbf{Stability does not guarantee the best extrapolating schedule.}
% The Copy task shows the limitation. 
% RL-Halting reduces OOD standard deviation from $13.2$ to $2.7$, but its average OOD accuracy and extrapolation frontier are worse than the length-matched schedule. 
% Thus, a learned stochastic schedule can stabilise the model around a suboptimal computation. 
% This reinforces the message of \cref{sec:stochastic-schedules}: stabilising depth selection is useful, but it is not sufficient to guarantee that the selected computation extrapolates.

\textbf{Stability does not guarantee the best extrapolating schedule.}
On Copy, RL-Halting reduces oracle OOD standard deviation
from $13.2$ to $2.7$ relative to $K=n$, but has lower
oracle accuracy and a shorter extrapolation frontier.
Its small $1.9$-point gap coexists with only $41.3\%$
policy accuracy, compared with $54.4\%$ for $K=n\pm5$.
While $K=n$ achieves the highest oracle accuracy,
$K=n\pm5$ achieves the highest policy accuracy.
This reinforces the message of \cref{sec:stochastic-schedules}: stabilising depth selection is useful, but it is not sufficient to guarantee that the selected computation extrapolates.
% Thus, low variability and a small oracle--policy gap
% do not ensure strong extrapolation.

\begin{observation}
\paragraph{Adaptive stopping improves the trade-off, but does not solve selection.}
Learning the stopping distribution can improve both extrapolation and stability on several tasks, suggesting that the benefits of stochastic schedules can be made adaptive. 
However, low variance and high extrapolation accuracy remain distinct objectives: a stable schedule can still select a non-extrapolating computation.
\end{observation}

\section{Why Does Randomisation Stabilise Looping?} 
\label{sec:toy-analysis} 

We now give an informal explanation for the two stabilising effects observed above: lower run-to-run OOD variability and lower sensitivity to the inference-time loop count.
The goal is not to prove improved extrapolation, but to identify two possible mechanisms: gradient averaging and loop-trajectory consistency.
Formal toy statements and proofs are given in \cref{app:randomized-schedule-theory,app:loop-consistency-theory}.
% The previous section showed that randomising the training loop depth stabilises Looped Transformers in two ways: it reduces run-to-run OOD variability, and it makes predictions less sensitive to the inference-time loop count. We now give an informal explanation for these effects. The goal is not to prove that randomisation improves extrapolation, but to identify two simple mechanisms by which it can reduce fragility. Formal toy statements and proofs are given in \cref{app:randomized-schedule-theory,app:loop-consistency-theory}.

\subsection{Why can randomisation reduce fragility? A toy gradient-averaging view}
\label{sec:toy-gradient-averaging}

\begin{wrapfigure}{r}{0.5\textwidth}
  \vspace{-5pt}
  \centering

    \resizebox{\linewidth}{!}{%
    \begin{tikzpicture}[
        vector/.style={->, -{Stealth[scale=1.1]}, thick},
        faint_vector/.style={->, -{Stealth[scale=0.7]}, thin, crimson_red, opacity=0.23},
        faint_residual/.style={->, -{Stealth[scale=0.65]}, dashed, thin, noise_amber, opacity=0.45},
        contour/.style={draw=gray!30, thick},
        smalllabel/.style={font=\small, text=slate_grey}
    ]

    \draw[contour, fill=bg_contour]
        plot [smooth cycle, tension=0.7]
        coordinates {(-0.8,-0.35) (4.8,-0.35) (5.0,3.4) (2.5,4.2) (-0.9,3.0)};

    \coordinate (O) at (0,0);
    \coordinate (Gbar) at (2.25,2.15);
    \coordinate (Gavg) at ($(Gbar)+(-8:0.25cm)$);
    \coordinate (Gk) at ($(Gbar)+(-40:1.55cm)$);

    \fill[slate_grey] (O) circle (2.3pt)
        node[below=4pt, font=\normalsize] {$\theta$};

    \draw[draw=slate_grey!60, fill=white, fill opacity=0.45, dashed, thick]
        (Gbar) circle (1.55cm);
    \node[smalllabel, anchor=south east] at ($(Gbar)+(130:1.65cm)$)
        {residual bound $B$};

    \draw[draw=navy_blue!60, fill=navy_blue!10, fill opacity=0.75, dashed, thick]
        (Gbar) circle (0.35cm);

    \foreach \angle/\radius in {20/1.35, 60/1.45, 105/1.25, 150/1.55, 205/1.30, 250/1.40, 295/1.20} {
        \coordinate (Grand) at ($(Gbar)+(\angle:\radius)$);
        \draw[faint_vector] (O) -- (Grand);
        \draw[faint_residual] (Gbar) -- (Grand);
    }

    \draw[vector, crimson_red, line width=1.35pt] (O) -- (Gk)
        node[
            pos=0.65, 
            below right=0pt,
            inner sep=1pt,
            text=crimson_red,
            font=\normalsize
        ] {$G_k$};

    \draw[vector, noise_amber, dashed, line width=1.2pt] (Gbar) -- (Gk)
        node[
            midway,
            sloped,
            above=0pt,
            inner sep=2pt,
            text=noise_amber,
            font=\normalsize
        ] {$\xi_k$};

    \draw[vector, navy_blue, line width=1.8pt] (O) -- (Gavg)
        node[
            pos=0.6, 
            right=8pt, 
            inner sep=0pt,
            text=navy_blue,
            font=\normalsize\bfseries
        ] {$\mathbb{E}[G_K]$};

    \draw[vector, slate_grey, line width=1.45pt] (O) -- (Gbar)
        node[
            pos=0.5, 
            above left=1pt,
            inner sep=0pt,
            text=slate_grey,
            font=\normalsize\bfseries
        ] {$\bar G$};

    \node[
        draw=gray!35,
        fill=white,
        fill opacity=0.96,
        text opacity=1,
        rounded corners=3pt,
        inner sep=4pt,
        anchor=north, 
        text=slate_grey
    ] at (2.0, -0.8) { 
        \begin{tabular}{@{}ll@{}}
        \textcolor{crimson_red}{\small\textbf{Single depth:}}
        &
        $\small \textcolor{crimson_red}{G_k}=\textcolor{slate_grey}{\bar G}+\textcolor{noise_amber}{\xi_k}$
        \\[0.1em]
        \textcolor{navy_blue}{\small\textbf{Randomised:}}
        &
        $\small \textcolor{navy_blue}{\mathbb{E}[G_K]}
        =
        \textcolor{slate_grey}{\bar G}+\textcolor{noise_amber}{\mathbb{E}[\xi_K]}\approx\textcolor{slate_grey}{\bar G}$
        \end{tabular}
    };

    \end{tikzpicture}
    }
    \caption{
\textbf{Toy geometry of gradient averaging.}
For an example $z$, $\textcolor{crimson_red}{G_k(\theta;z)}=\nabla_\theta L_k(\theta;z)$ is the gradient used to update the looped model parameters $\theta$ when training stops after $k$ loop iterations. 
We write $\textcolor{crimson_red}{G_k}=\textcolor{slate_grey}{\bar G}+\textcolor{noise_amber}{\xi_k}$, where $\textcolor{slate_grey}{\bar G}$ is the component common to gradients from nearby depths and $\textcolor{noise_amber}{\xi_k}$ is a depth-specific residual. 
A single-depth update can follow one residual direction, while sampling $K$ from nearby depths averages these residuals, moving $\textcolor{navy_blue}{\mathbb{E}_K[G_K]}$ closer to the shared component $\textcolor{slate_grey}{\bar G}$.
}
% \caption{
% Toy geometry of gradient averaging. 
% A single-depth update $\textcolor{crimson_red}{G_k}$ can include a depth-specific residual $\xi_k$, while randomising over nearby depths averages these residuals and moves the expected update $\textcolor{navy_blue}{\mathbb{E}[G_K]}$ closer to the shared component $\textcolor{slate_grey}{\bar G}$.
% \patrik{add a bit more context/refer to main text: I guess G is updating the weights (which weights in particular? what does k refer to? depth?) also if the reader only reads the caption, they won't know what a shared component is}}
    % \caption{
    % Toy geometry of gradient averaging. 
    % A single-depth update $\textcolor{crimson_red}{G_k}$ can deviate from the shared direction $\textcolor{slate_grey}{\bar G}$ by a depth-specific residual $\textcolor{noise_amber}{\xi_k}$. 
    % When randomised schedules sample several nearby loop depths, weakly aligned residuals partially cancel in expectation, moving the expected update $\textcolor{navy_blue}{\mathbb{E}[G_K]}$ closer to the shared component $\textcolor{slate_grey}{\bar G}$. 
    % This provides a possible stabilising mechanism, but does not imply that $\textcolor{slate_grey}{\bar G}$ is extrapolating.
    % }
    \label{fig:gradient-averaging-illustration}
% \vspace{-20pt}
\end{wrapfigure}

Although each SGD step samples only one loop depth, the expected gradient corresponds to a depth-averaged objective. 
Let $L_k(\theta;z)$ be the loss on example $z=(x,y)$ after $k$ loop iterations, and write
$
G_k(\theta;z)=\nabla_\theta L_k(\theta;z)
$
for the corresponding depth-$k$ gradient. 
Define
\begin{align}
J_w(\theta;z)
=
\frac{1}{|S_w(x)|}
\sum_{k\in S_w(x)} L_k(\theta;z),
\end{align}
where $S_w(x)$ is a local window of loop depths around the input length. 
If $K\sim\mathrm{Unif}(S_w(x))$, then
\begin{align}
\mathbb{E}_{K}\big[G_K(\theta;z)\big]
=
\nabla_\theta J_w(\theta;z).
\end{align}
Thus, randomisation changes the expected gradient from a single-depth gradient to an average over nearby loop depths.
\cref{fig:gradient-averaging-illustration} illustrates this view: a single-depth gradient can contain a shared component and a depth-specific residual, while randomising over nearby depths averages these residuals in expectation.

To interpret this averaging effect, write each depth-specific gradient as
$
\textcolor{crimson_red}{G_k}=\textcolor{slate_grey}{\bar G}+\textcolor{noise_amber}{\xi_k},
$
where $\textcolor{slate_grey}{\bar G}$ is the component shared across nearby depths and $\textcolor{noise_amber}{\xi_k}$ is a depth-specific residual. 
If the residuals are bounded and weakly aligned, averaging over loop depths suppresses their contribution, so the expected update $\textcolor{navy_blue}{\mathbb{E}[G_K]}$ moves closer to the shared direction $\textcolor{slate_grey}{\bar G}$. 
This gives a possible mechanism for the reduced run-to-run spread in \cref{fig:random-window-ood-summary}. 
A formal version of this toy argument is given in \cref{prop:gradient-averaging}.
A controlled intervention supports this interpretation:
replacing sampled-depth supervision with explicit
multi-depth averaging also reduces run-to-run OOD
variability, showing that random depth selection is
not necessary for stabilisation in this setting
(\cref{app:multi-depth-supervision}).

This mechanism explains stabilisation, not extrapolation. 
The shared direction $\textcolor{slate_grey}{\bar G}$ need not be the direction that leads to the best OOD solution, so lower variance need not imply higher mean OOD accuracy.

\subsection{Why can randomisation stabilise the loop trajectory?}
\label{sec:toy-loop-consistency}

A complementary view is that randomised schedules supervise the same input at multiple nearby loop depths. 
A deterministic schedule only asks the model to be correct at one prescribed depth, whereas a randomised schedule asks it to produce compatible answers across a local window of depths.

This can stabilise the loop trajectory. 
If nearby loop states are trained to support the same target output, then predictions are less likely to change abruptly from one iteration to the next. 
This gives an interpretation of the lower prediction flip rates and broader useful-$K$ regions in \cref{fig:loop-trajectory-stability}. 
Again, this is only a stabilising mechanism: it encourages consistency across nearby depths, but does not guarantee that the consistent computation is the extrapolating one. 
A formal sufficient-condition argument is given in \cref{prop:loop-consistency}.
The same intervention yields a small oracle--policy
gap, consistent with reduced sensitivity to stopping
depth (\cref{app:multi-depth-supervision}).

\section{{Related Works}}
\textbf{Stability of length generalisation.}
Prior work has shown that Transformer length generalisation can be possible but fragile. 
Most directly, \citet{zhou_transformers_2024} show that extrapolation on integer addition is highly sensitive to random initialisation and training data order, even with strong in-distribution performance. 
Other work connects length-generalisation failures to positional-encoding distribution shift~\citep{ruoss_randomized_2023,cho_position_2024,kazemnejad_impact_nodate}, attention-variance mechanisms~\citep{liVANISHINGVARIANCETRANSFORMER2025}, and broader effects of initialisation or data ordering on optimisation~\citep{ramasingheHowMuchDoes2023,hacohenPowerCurriculumLearning2019,schwarzschild_can_2021}. 
Our work studies a complementary instability specific to looped architectures: the loop count used during training is not merely a compute budget, but affects which computation trajectory is selected and how stable length extrapolation becomes.

\textbf{Extrapolation and recurrent structure.}
Recurrent and weight-tied architectures support extrapolation by reusing computation across multiple steps, as studied in recurrent networks, Universal Transformers, deep equilibrium models, and Looped Transformers~\citep{schwarzschild_can_2021,bansal_end--end_2022,dehghani_universal_2018,anil_path_nodate,giannou_looped_2023,yang_looped_2023,fan_looped_2024,saunshi_reasoning_2024,geiping_scaling_2025,zhu_scaling_2025,wang_hierarchical_2025}. 
In particular, \citet{fan_looped_2024} show that Looped Transformers can achieve strong length generalisation on arithmetic tasks, demonstrating the extrapolation potential of repeated computation. 
Our work studies a complementary question: whether this potential is selected reliably during training. 
We show that looped models can exhibit high run-to-run OOD instability, and that stochastic loop schedules can regularise this instability by reducing dependence on a single prescribed loop depth.

\textbf{Learning when to stop.}
Adaptive-computation methods learn input-dependent computation
budgets, including Adaptive Computation Time
\citep{graves_adaptive_2017}, PonderNet
\citep{banino_pondernet_2021}, and learned-depth Looped
Language Models \citep{zhu_scaling_2025}.
The stopping rule can affect reasoning quality because both
insufficient computation and overthinking can hurt performance
\citep{bansal_thinking_2021,bansal_end--end_2022}.
RL-Halting uses a related probabilistic stopping
parameterisation, but supervises one sampled depth per example
and updates the stopping head using the realised task loss.
Our PonderNet-style baseline instead weights task losses
across depths by the stopping distribution
(\cref{app:pondernet-comparison}).
This comparison examines how alternative ways of training
the stopping distribution affect extrapolation and
run-to-run stability.
% Adaptive-computation methods learn to allocate computation across inputs, from Adaptive Computation Time~\citep{graves_adaptive_2017} to probabilistic halting methods such as PonderNet~\citep{banino_pondernet_2021}. 
% In recurrent reasoning, the stopping rule matters because too few iterations may under-compute, while too many can cause overthinking and hurt extrapolation~\citep{bansal_thinking_2021,bansal_end--end_2022}. 
% Recent Looped Language Models also use learned depth allocation and relate it to PonderNet-style objectives~\citep{zhu_scaling_2025}. 
% Our RL-Halting variant is related to these methods, but matches our sampled-depth protocol: each update samples one stopping depth and supervises only that realised depth. 
% This lets us study learned stopping as an adaptive form of stochastic loop-depth regularisation, rather than primarily as compute saving.
% Appx.~\ref{app:pondernet-comparison} compares RL-Halting with a PonderNet-style baseline.

\textbf{Randomisation as regularisation.}
Randomisation is widely used to regularise neural networks by preventing over-reliance on a single computation path. 
Dropout randomly removes units~\citep{srivastavaDropoutSimpleWay}, stochastic depth randomly skips residual layers~\citep{huangDeepNetworksStochastic2016}, and LayerDrop applies structured dropout to Transformer layers~\citep{fanReducingTransformerDepth2019}. 
Randomisation has also been used to improve length generalisation, for example through randomised positional encodings~\citep{ruoss_randomized_2023}, and to make equilibrium-style models less sensitive to unrolling depth through randomised-depth training~\citep{anil_path_nodate}. 
Our work studies a different axis of randomisation: the loop depth at which a recurrent-depth model is supervised. 
This can be viewed as regularising the learned computation trajectory by averaging over nearby stopping depths, thereby reducing sensitivity to a single prescribed loop count.

\section*{Conclusion}
\label{sec:conclusion}
We studied length generalisation in Looped Transformers through stopping schedules. 
Deterministic schedules reveal extrapolation potential but select it unreliably across runs, creating a training-time selection problem. 
Stochastic schedules improve run-to-run stability and robustness to inference-time loop count by training across nearby depths. 
RL-Halting often improves this accuracy--stability trade-off, but can also stabilise a suboptimal computation. 
Thus, ``when to stop'' is a central training-time design choice in looped architectures.

\paragraph{Limitations and future work.}
Our controlled algorithmic setting enables us to isolate how
training stopping schedules affect the reliability of length
extrapolation; their impact on pretrained models and realistic
reasoning tasks remains to be established.
The distinction between stabilisation and extrapolation
suggests that future stopping objectives should account for
which computations generalise, beyond maintaining stochasticity
or minimising in-distribution loss.
Promising directions include incorporating priors that favour
reusable computation and combining learned stochastic stopping
with architectural stabilisation.
Evaluating these approaches on richer tasks would clarify
when training-time depth selection supports both reliable
extrapolation and effective use of test-time computation.
% We studied length generalisation in Looped Transformers through the lens of stopping schedules. 
% Our results show that looping does not merely add computation; it creates a training-time selection problem over computation trajectories. 
% Deterministic schedules can reveal strong extrapolation potential, but this potential is unstable across runs. 
% Stochastic schedules reduce this fragility by training across nearby loop depths, improving both run-to-run stability and robustness to the inference-time loop count. 
% However, stability alone does not guarantee extrapolation: RL-Halting often improves the accuracy--stability trade-off, but can also stabilise a suboptimal computation. 
% These findings suggest that ``when to stop'' should be treated as a central training-time design choice in looped architectures. 

\section*{Author Contributions}
H.-Y.K. developed the experimental framework, implemented the methods, conducted the experiments, analysed the results, and led the writing of the manuscript. The project originated in W.B.'s group, with P.R. and W.B. playing central roles in the initial brainstorming, defining the research direction, and shaping the experimental agenda. P.R. provided close mentoring throughout the project, including guidance on method development, experimental design, interpretation of results, and manuscript preparation. After the project continued in M.J.'s group at EPFL, E.M.C. also provided close mentoring and regular guidance on the research direction, experimental design, manuscript preparation, and possible methodological extensions, including reinforcement-learning-based improvements. M.J. provided supervisory feedback, as well as institutional and computational support at EPFL.

\begin{ack}
H.-Y.K. initiated this project during an internship in W.B.'s group at the Max Planck Institute for Intelligent Systems and the ELLIS Institute, T\"ubingen, and continued it as a project student in M.J.'s group at EPFL. We gratefully acknowledge the EPFL Research Computing Platform (RCP) for providing computational resources through M.J.'s group. P.R. acknowledges his membership in the European Laboratory for Learning and Intelligent Systems (ELLIS) PhD program and thanks the International Max Planck Research School for Intelligent Systems (IMPRS-IS) for its support. This work was supported by the German Federal Ministry of Education and Research (BMBF) through the T\"ubingen AI Center, FKZ: 01IS18039A, and by a grant from Coefficient Giving to Aaron Mueller. W.B. acknowledges financial support from an Emmy Noether Grant funded by the German Research Foundation (DFG) under grant no.\ BR 6382/1-1, and from the Open Philanthropy Foundation, funded by the Good Ventures Foundation. W.B. is a member of the Machine Learning Cluster of Excellence, EXC number 2064/1 -- Project number 390727645. This research also used compute resources at the T\"ubingen Machine Learning Cloud, DFG FKZ INST 37/1057-1 FUGG. Finally, we thank the organisers and participants of the Fourth Bellairs Workshop
on Causality, held at the McGill University Bellairs Research Institute
from 14--21 February 2025, for helpful discussions that contributed to
the development of the project idea, particularly Kartik Ahuja.
We also thank Sajad Movahedi, Vera Milovanovi\'{c}, and Ting-Wen Ko
for valuable discussions.
\paragraph{Competing interests.}
The authors declare no competing financial interests related to this work.
% Use unnumbered first level headings for the acknowledgments. All acknowledgments

% Do {\bf not} include this section in the anonymized submission, only in the final paper. You can use the \texttt{ack} environment provided in the style file to automatically hide this section in the anonymized submission.
\end{ack}
\newpage

\bibliographystyle{unsrtnat}
\bibliography{my_library}
% \section*{References}

% References follow the acknowledgments in the camera-ready paper. Use unnumbered first-level heading for
% the references. Any choice of citation style is acceptable as long as you are
% consistent. It is permissible to reduce the font size to \verb+small+ (9 point)
% when listing the references.
% Note that the Reference section does not count towards the page limit.
% \medskip

% {
% \small

% [1] Alexander, J.A.\ \& Mozer, M.C.\ (1995) Template-based algorithms for
% connectionist rule extraction. In G.\ Tesauro, D.S.\ Touretzky and T.K.\ Leen
% (eds.), {\it Advances in Neural Information Processing Systems 7},
% pp.\ 609--616. Cambridge, MA: MIT Press.

% [2] Bower, J.M.\ \& Beeman, D.\ (1995) {\it The Book of GENESIS: Exploring
%   Realistic Neural Models with the GEneral NEural SImulation System.}  New York:
% TELOS/Springer--Verlag.

% [3] Hasselmo, M.E., Schnell, E.\ \& Barkai, E.\ (1995) Dynamics of learning and
% recall at excitatory recurrent synapses and cholinergic modulation in rat
% hippocampal region CA3. {\it Journal of Neuroscience} {\bf 15}(7):5249-5262.
% }

%%%%%%%%%%%%%%%%%%%%%%%%%%%%%%%%%%%%%%%%%%%%%%%%%%%%%%%%%%%%

\appendix
\newpage

\crefalias{section}{appendix}
\crefalias{subsection}{appendix}

% \section{Technical appendices and supplementary material}
% Technical appendices with additional results, figures, graphs, and proofs may be submitted with the paper submission before the full submission deadline (see above). You can upload a ZIP file for videos or code, but do not upload a separate PDF file for the appendix. There is no page limit for the technical appendices. 

% Note: Think of the appendix as ``optional reading'' for reviewers. The paper must be able to stand alone without the appendix; for example, adding critical experiments that support the main claims to an appendix is inappropriate. 

\section{Implementation Details}
\label{app:implementation-details}

\paragraph{Architecture and training.}
Our Looped Transformer implementation follows Fan et al.~\cite{fan_looped_2024}. 
One looped block consists of three Transformer layers and is reused across loop iterations with input injection at each step. 
Following the length-generalisation setup of Fan et al.~\cite{fan_looped_2024}, we use no positional embeddings (NoPE). 
The 3-layer Transformer baseline matches the parameter count of one looped block, while the 60-layer Transformer baseline provides a higher-compute fixed-depth comparison. 

All models are trained with AdamW using batch size 64 for 100k gradient steps. 
We use an initial learning rate of $10^{-4}$ and decay it to $0$ with a cosine schedule over training.
For stability evaluation, we run $32$ independent seeds on Addition and $18$ independent seeds on the other tasks.

\paragraph{Stopping schedules.}
For stochastic schedules, each optimisation step samples one loop count per example. 
For the randomised length-matched schedules, we sample
\begin{align}
\Delta \sim \mathrm{Unif}\{-w,\ldots,w\},
\qquad
K=\mathrm{clip}(n+\Delta,1,T_{\max}).
\end{align}

\paragraph{Loop-depth computation during training.}
During training, all Looped Transformer runs use a maximum loop budget of $T_{\max}=20$. 
For the fixed-depth schedule $K=c$, the model is unrolled for exactly $c$ loop iterations on every batch, and the loss is computed from the representation at depth $c$. 
For the length-matched schedule $K=n$, examples in the same batch may have different lengths. 
Rather than unrolling each example separately, we unroll the whole batch up to the maximum target depth in that batch, $\max_i n_i$, and compute the loss for each example $i$ using its corresponding representation at depth $n_i$. 

For the randomised length-matched schedule, each example samples a depth
$
K_i=\mathrm{clip}(n_i+\Delta_i,1,T_{\max}).
$
We then unroll the batch up to $\max_i K_i$ and compute each example's loss using the representation at its sampled depth $K_i$. 
For RL-Halting, we unroll the model up to $T_{\max}=20$ in order to construct the full stopping distribution over depths. 
A stopping time $\tau_i$ is sampled for each example, and the task loss is computed only from the representation at the sampled depth $\tau_i$. 
Thus, all stochastic schedules use one supervised depth per example per update, but differ in whether that depth is fixed by hand, sampled from a hand-designed distribution, or sampled from a learned stopping distribution.

\paragraph{Evaluation.}
Unless otherwise stated, OOD summaries are computed over lengths $\{20,25,30,\ldots,60\}$, and a prediction is correct only if the complete output sequence is correct, including the end token.
Oracle search depths and policy stopping rules are specified
in \cref{sec:metrics}.
Full oracle and policy OOD results and their aggregation
are reported in \cref{app:oracle-policy-gap}.

\paragraph{Controlled variance experiments.}
For the ``Original'' setting, both random initialisation and data order vary across runs. 
For ``Fixed init.'', all runs use the same initialisation and vary only the data order. 
For ``Fixed data'', all runs use the same data order and vary only the initialisation.

\paragraph{Compute resources.}
All experiments can be run on a single NVIDIA A100 GPU with 80GB memory, without model or data parallelism. 
A single Looped Transformer training run typically takes less than four hours, depending on the task and stopping schedule; fixed-depth Transformer baselines are generally cheaper. 
The total compute scales with the number of independent seeds reported above: $32$ seeds for Addition and $18$ seeds for each of the other tasks.

\subsection{Training compute comparison}
\label{app:compute-comparison}

We compare the training computation of the four stopping
schedules using the same Looped Transformer architecture,
batch size $64$, and maximum loop budget $T_{\max}=20$.
Each schedule applies task supervision at one depth per
example per update, but the schedules differ in how far
each training batch is unrolled.

\begin{table}[t]
  \centering
  \small
  \caption{
  \textbf{Training computation across stopping schedules.}
  One block evaluation denotes one application of the shared
  three-layer Transformer block to the whole batch.
  Average counts are reported over $100{,}000$ training steps
  with batch size $64$ and $T_{\max}=20$.
  }
  \label{tab:training-compute}
  \begin{tabular}{@{}lp{0.49\textwidth}r@{}}
    \toprule
    Schedule & Batch unrolling during training
    & \shortstack{Avg.\ block\\evaluations} \\
    \midrule
    Fixed $K=20$
    & Always unrolled for $20$ iterations.
    & $20.00$ \\[3pt]
    $K=n$
    & Unrolled to the largest problem length
      $\max_i n_i$ in the batch.
    & $18.97$ \\[3pt]
    $K=n\pm5$
    & Unrolled to the largest sampled depth
      $\max_i K_i$ in the batch.
    & $19.99$ \\[3pt]
    RL-Halting
    & Always unrolled to $T_{\max}=20$ to construct
      the full stopping distribution.
    & $20.00$ \\
    \bottomrule
  \end{tabular}
\end{table}

For the randomised schedule, each example samples
$K_i=\mathrm{clip}(n_i+\Delta_i,1,T_{\max})$, where
$\Delta_i\sim\mathrm{Unif}\{-5,\ldots,5\}$.
Because examples are processed together, recurrent forward
computation depends on the largest requested depth in the
batch, rather than the average selected depth across examples.

\paragraph{Recurrent computation and stopping-head overhead.}
RL-Halting has Transformer forward computation comparable
to the fixed $K=20$ schedule, with additional overhead from
evaluating a single linear stopping head at each depth and
updating its parameters.
The stopping head operates on detached pooled hidden states,
and task loss is computed only at one sampled depth per example.
In contrast, $K=n$ and $K=n\pm5$ can terminate batch
unrolling earlier.

As shown in \cref{tab:training-compute}, RL-Halting uses
approximately $5.4\%$ more block evaluations than $K=n$
and approximately $0.05\%$ more than $K=n\pm5$, based
on the rounded averages.
Thus, at the batch size used here, the randomised schedule
has nearly the same recurrent forward computation as
RL-Halting despite permitting earlier termination.
These counts measure recurrent forward computation,
rather than total training FLOPs; they exclude stopping-head
operations and backward computation.

\paragraph{Wall-clock comparison.}
In a matched wall-clock comparison using $50{,}000$
optimisation steps, batch size $64$, and the same execution
environment, all four schedules trained in approximately
$2.5$ hours, with differences within roughly $10$ minutes.
These timing runs use a shorter training duration than the
$100{,}000$-step block-count comparison above.
The measurements include the stopping-head overhead and
indicate that the practical training cost of RL-Halting
is comparable to the other schedules in our setting.

RL-Halting is not designed as an early-exit or compute-saving
method. Its purpose is to learn the training stopping
distribution rather than specify it by hand.
The hand-designed schedules can require fewer block
evaluations, while the observed wall-clock differences
remain small under our experimental configuration.

\section{RL-Halting Details}
\label{app:rl-halting}

We describe the learned stochastic stopping rule used in \cref{sec:rl-halting-formulation,sec:loss_rl_halting}. 
RL-Halting is related to adaptive-computation and probabilistic-halting methods such as Adaptive Computation Time and PonderNet~\citep{graves_adaptive_2017,banino_pondernet_2021}, as well as recent learned-depth looped models~\citep{zhu_scaling_2025}. 
Unlike objectives that train a distribution-weighted loss over all depths, RL-Halting follows our sampled-depth protocol: it samples one depth during training, supervises only at the realised depth, and updates the stopping policy from the realised task loss.

\paragraph{Stopping distribution.}
Let $H^{(t)} \in \mathbb{R}^{m \times d}$ denote the hidden states after $t$ loop iterations, and let $T_{\max}$ be the maximum allowed number of loop iterations. After each iteration, we average-pool the hidden states to obtain a summary vector
\begin{equation}
h_t = \mathrm{Pool}(H^{(t)}).
\end{equation}
Following a hazard-style parameterisation similar to PonderNet~\citep{banino_pondernet_2021}, a learned stopping head $q_\phi$ maps this summary to a stopping hazard
\begin{equation}
r_t = \sigma(q_\phi(h_t)).
\end{equation}
where $\sigma$ is the sigmoid function. The hazard $r_t$ is the probability of stopping at iteration $t$ conditioned on not having stopped earlier. This induces the stopping distribution
\begin{equation}
\pi_{\mathrm{stop},\phi}(\tau=t \mid x)
=
r_t \prod_{i=1}^{t-1}(1-r_i),
\qquad t < T_{\max}.
\end{equation}
The remaining probability mass is assigned to the final iteration:
\begin{equation}
\pi_{\mathrm{stop},\phi}(\tau=T_{\max} \mid x)
=
\prod_{i=1}^{T_{\max}-1}(1-r_i).
\end{equation}
Thus, the model is forced to stop by $T_{\max}$.

\paragraph{Sampled-depth supervision.}
For each training example $(x,y)$, we sample a stopping time
\begin{equation}
\tau \sim \pi_{\mathrm{stop},\phi}(\cdot \mid x).
\end{equation}
The Looped Transformer is then supervised only at the sampled depth $\tau$. Let $\hat{y}_\tau(x)$ denote the prediction after $\tau$ loop iterations. The task loss is
\begin{equation}
\ell(\theta; x,y,\tau)
=
\mathrm{CE}(\hat{y}_\tau(x), y),
\end{equation}
where $\mathrm{CE}$ denotes the task cross-entropy loss. The model parameters $\theta$ are updated using the ordinary supervised gradient of this sampled-depth loss.

\paragraph{Stopping-policy update.}
The stopping policy is trained with a score-function estimator, also known as REINFORCE~\citep{williams_simple_nodate,murphy_reinforcement_2025}. 
We define the reward as the negative task loss:
\begin{equation}
R(x,y,\tau) = -\ell(\theta; x,y,\tau).
\end{equation}
The objective for the stopping policy is to maximise the expected reward under the stopping distribution:
\begin{equation}
J(\phi)
=
\mathbb{E}_{\tau \sim \pi_{\mathrm{stop},\phi}(\cdot \mid x)}
[
R(x,y,\tau)
].
\end{equation}
Using the score-function identity, the policy-gradient estimator is
\begin{equation}
\nabla_\phi J(\phi)
=
\mathbb{E}_{\tau \sim \pi_{\mathrm{stop},\phi}(\cdot \mid x)}
[
(R(x,y,\tau)-b)
\nabla_\phi \log \pi_{\mathrm{stop},\phi}(\tau \mid x)
],
\end{equation}
where $b$ is a moving-average baseline used to reduce gradient variance. In practice, we also add an entropy regulariser over the stopping distribution to encourage exploration over loop depths:
\begin{equation}
J_{\mathrm{ent}}(\phi)
=
J(\phi)
+
\lambda_{\mathrm{ent}}
\mathcal{H}
(
\pi_{\mathrm{stop},\phi}(\cdot \mid x)
).
\end{equation}

\paragraph{Training loss.}
Equivalently, when using gradient descent, we minimise the stopping-policy loss
\begin{equation}
\mathcal{L}_{\mathrm{halt}}(\phi)
=
-
\mathrm{sg}(R-b)
\log \pi_{\mathrm{stop},\phi}(\tau \mid x)
-
\lambda_{\mathrm{ent}}
\mathcal{H}
(
\pi_{\mathrm{stop},\phi}(\cdot \mid x)
),
\end{equation}
where $\mathrm{sg}(\cdot)$ denotes stop-gradient. The total update consists of the sampled-depth task loss for the Looped Transformer and the policy-gradient loss for the stopping head:
\begin{equation}
\mathcal{L}
=
\ell(\theta; x,y,\tau)
+
\mathcal{L}_{\mathrm{halt}}(\phi).
\end{equation}

\paragraph{Relation to stochastic schedules.}
The hand-designed stochastic schedules in Sec.~\ref{sec:stochastic-schedules} sample loop depths from a fixed distribution, such as a local window around the input length. RL-Halting instead learns this distribution from the realised task loss. Both approaches expose the model to sampled stopping depths during training, but RL-Halting makes the sampling distribution adaptive.

\paragraph{Training, test-time horizons.}
% During training, RL-Halting samples stopping times from a truncated hazard distribution with horizon $T_{\mathrm{train}}=20$. 
Unless otherwise stated, RL-Halting uses a training horizon
of $T_{\mathrm{train}}=20$.
We vary this horizon in \cref{app:depth-sensitivity}
to examine its effect on extrapolation.
The remaining probability mass is assigned to the final training iteration $T_{\mathrm{train}}$, ensuring that every training example has a valid sampled stopping depth. 
At test time, because the Transformer block and stopping head are shared across loop iterations, the same hazard rule can be evaluated over a different finite horizon. 
For policy-based evaluation in our experiments, we use $T_{\mathrm{test}}=30$.

More generally, for any chosen horizon $T$, the truncated stopping distribution is
\begin{align}
\pi_{\mathrm{stop},\phi}^{(T)}(\tau=t\mid x)
=
r_t\prod_{i=1}^{t-1}(1-r_i),
\qquad t<T,
\end{align}

with the remaining tail mass assigned to the final depth,
\begin{align}
\pi_{\mathrm{stop},\phi}^{(T)}(\tau=T\mid x)
=
\prod_{i=1}^{T-1}(1-r_i).
\end{align}

Thus, $T$ is a truncation horizon for sampling or analysis, not a separate learned parameter. 
The hazards $r_t$ are produced by the same shared stopping head at each loop depth; changing $T$ only changes where the remaining tail probability is placed.

% \paragraph{Implementation details.}
% The stopping head $q_\phi$ is implemented as a single linear layer mapping the pooled hidden state $h_t$ to a scalar logit. 
% The moving-average baseline is updated with decay $0.9$ using the minibatch mean reward:
% $
% b \leftarrow 0.9\,b + 0.1\,\bar R,
% $
% where $\bar R$ is the average reward in the current minibatch. 
% We set the entropy coefficient to $\lambda_{\mathrm{ent}}=0.01$ in all RL-Halting experiments. 
% For the halting-policy loss $\mathcal{L}_{\mathrm{halt}}$, the pooled hidden states are detached before computing the stopping logits, so this loss updates only the stopping head parameters $\phi$; the Looped Transformer parameters $\theta$ are updated through the sampled-depth task loss $\ell(\theta;x,y,\tau)$.

\paragraph{Implementation details.}
The stopping head $q_\phi$ is a single linear layer mapping
the pooled hidden state $h_t$ to a scalar logit.
We update the moving-average reward baseline as
$b \leftarrow 0.9\,b + 0.1\,\bar R$,
where $\bar R$ is the current minibatch mean reward.
Unless otherwise stated, we use an entropy coefficient of
$\lambda_{\mathrm{ent}}=0.001$; sensitivity to this coefficient
is examined in \cref{app:entropy-sensitivity}.
The pooled hidden states are detached before computing the
stopping logits, so the halting-policy loss
$\mathcal{L}_{\mathrm{halt}}$ updates only the stopping-head
parameters $\phi$.
The Looped Transformer parameters $\theta$ are updated
through the sampled-depth task loss $\ell(\theta;x,y,\tau)$.

\section{Additional Results}
\label{app:additional-results}

\subsection{Controlled variance experiments}
\label{app:controlled-variance}

\begin{figure}[t]
  \centering
  \includegraphics[width=\columnwidth]{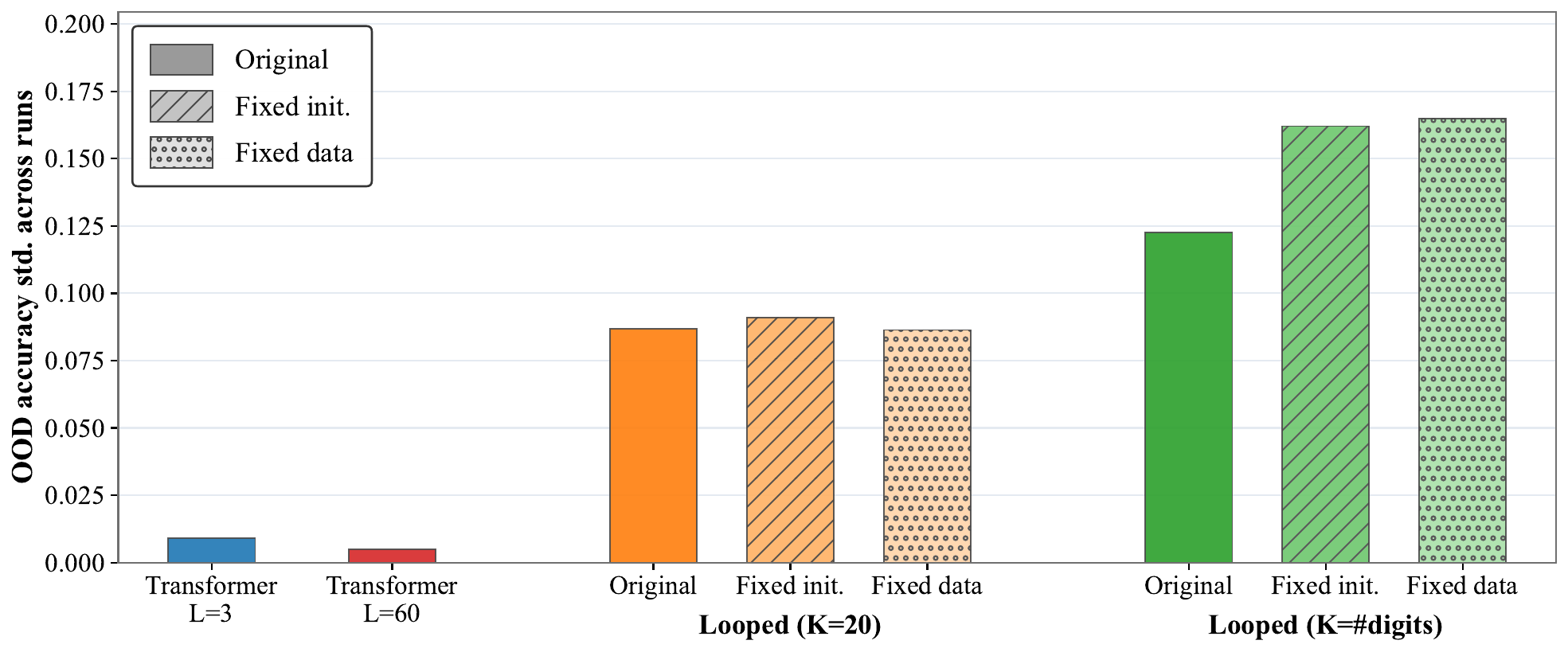}
  \caption{
  OOD variability on binary addition across models and controlled settings. 
  Bars show the mean standard deviation of oracle-over-iterations accuracy across runs, averaged over OOD digit lengths from $20$ to $60$. 
  For Looped Transformers, ``Original'' varies both initialisation and data order across runs, ``Fixed init.'' varies only data order, and ``Fixed data'' varies only initialisation. 
  Hatching marks the controlled settings.
  }
  \label{fig:ood-variance-comparison}
\end{figure}

To test whether OOD instability is caused mainly by random initialisation or by data order, we run two controlled variants of the looped experiments.
In the ``Fixed init.'' setting, all runs use the same initialisation and differ only in data order.
In the ``Fixed data'' setting, all runs use the same data order and differ only in initialisation.

\cref{fig:ood-variance-comparison} shows that the high OOD variance of Looped Transformers does not disappear under either control.
For fixed $K=20$, the OOD standard deviation remains around $0.09$ in both controlled settings, compared with $0.087$ in the original setting.
For the length-matched schedule $K=n$, the controlled standard deviations are even higher, around $0.16$, compared with $0.122$ in the original setting.
By contrast, the fixed-depth Transformer baselines have much lower OOD standard deviation, $0.009$ for the 3-layer model and $0.005$ for the 60-layer model.

These results suggest that OOD instability is not an artefact of varying all sources of randomness simultaneously.
Both initialisation and data order can steer looped training toward different extrapolation behaviours, even when models achieve similar in-distribution performance.

\subsection{Randomising around a fixed loop budget}
\label{app:fixed-k-random-ablation}

We also test whether stochasticity helps when the centre of the schedule is fixed rather than length-matched. 
Starting from the fixed schedule $K=20$, we sample
$
\Delta\sim\mathrm{Unif}\{-5,\ldots,5\}
$
and train at
$
K=20+\Delta.
$

\begin{table}[t]
\centering
\caption{
\textbf{Fixed-budget randomisation and learned-halting baselines on binary addition.}
OOD is average oracle-over-iterations accuracy over lengths $20$--$60$, Front.@90 is the mean-accuracy frontier at $90\%$, and Std. is the standard deviation across runs.
}
\label{tab:fixed-k-random-ablation}
\begin{tabular}{lccc}
\toprule
Setting & OOD $\uparrow$ & Front.@90 $\uparrow$ & Std. $\downarrow$ \\
\midrule
Fix $K=20$ & 34.2 & 25 & 7.4 \\
$K=20\pm5$ & 34.2 & 25 & 4.7 \\
PonderNet & 29.4 & 20 & 15.1 \\
RL-Halting & \textbf{45.0} & \textbf{30} & \textbf{2.7} \\
\bottomrule
\end{tabular}
\end{table}

\cref{tab:fixed-k-random-ablation} shows that randomising around the fixed budget leaves the mean OOD accuracy and extrapolation frontier unchanged, but reduces the standard deviation across runs from $7.4$ to $4.7$. 
This suggests that the stabilising effect of stochastic schedules is not specific to length-matched schedules: exposing the model to nearby loop depths can reduce run-to-run variability even when the schedule is centred at a fixed compute budget.

\subsection{An Interventional Test of the Toy Depth-Averaging Mechanism}
\label{app:multi-depth-supervision}

The toy analyses in
\cref{sec:toy-gradient-averaging,sec:toy-loop-consistency}
suggest that stabilisation arises from supervision across
nearby depths.
To test whether random depth sampling is necessary for this
effect, we compare stochastic single-depth training
($K=n\pm5$) with deterministic multi-depth supervision on
binary addition.
The latter averages the prediction losses over the same
depth window at each update, sharing the recurrent unroll
while evaluating the readout at each supervised depth.
Architecture, data generation, and optimisation settings
are otherwise held fixed. For policy evaluation, both models use the centre depth
$K=n$, without ground-truth depth selection.

\begin{table}[t]
  \centering
  \caption{
    \textbf{Multi-depth supervision on binary addition.}
    OOD accuracy and its standard deviation across runs,
    reported in percent.
    Deterministic multi-depth supervision yields lower
    variability, whereas stochastic sampling achieves
    higher mean accuracy.
  }
  \label{tab:multi-depth-supervision}
  \small
  \begin{tabular}{lcccc}
    \toprule
    & \multicolumn{2}{c}{Oracle OOD}
    & \multicolumn{2}{c}{Policy OOD} \\
    \cmidrule(lr){2-3}\cmidrule(lr){4-5}
    Supervision & Mean & Std. & Mean & Std. \\
    \midrule
    Stochastic single-depth & 30.9 & 5.9 & 29.6 & 5.2 \\
    Deterministic multi-depth & 27.6 & 3.9 & 26.2 & 3.7 \\
    \bottomrule
  \end{tabular}
\end{table}

As shown in \cref{tab:multi-depth-supervision},
deterministic multi-depth supervision reduces oracle OOD
standard deviation to $3.9\%$, compared with $5.9\%$ under
stochastic sampling and $9.9\%$ under single-depth $K=n$
training (\cref{tab:ood_generalization}).
Both multi-depth and stochastic supervision also yield
an oracle--policy gap of approximately $1.3$ percentage
points.
Thus, in this setting, supervision across nearby depths
is sufficient to improve stability without random
depth selection.

This intervention supports the depth-averaging mechanism:
stochastic training averages depth-specific gradients
in expectation, while multi-depth supervision averages
them explicitly at each update.
The small oracle--policy gaps are also consistent with
reduced sensitivity to stopping depth, although they
do not directly establish hidden-state consistency.
The results therefore provide causal evidence for the
stabilising effect of multi-depth supervision and
support, rather than uniquely identify, the proposed
toy mechanisms.

Finally, stochastic sampling achieves approximately
$3.3$ percentage points higher oracle and policy OOD
accuracy despite its greater variability.
This reinforces the distinction made by the toy analyses:
stabilisation does not guarantee selection of the
best extrapolating computation.

\subsection{Policy accuracy and oracle--policy gaps}
\label{app:oracle-policy-gap}

We report the full oracle and policy OOD results across
all four tasks, complementing the compact summary in
\cref{tab:ood_generalization}.
Policy evaluation follows the stopping rules in
\cref{sec:metrics}: $K=20$ for fixed-$20$ training,
$K=n$ for length-matched and randomised $K=n\pm5$ training,
and maximum-probability stopping with $T_{\mathrm{test}}=30$
for RL-Halting.
None of these rules uses ground-truth depth selection.

\paragraph{Aggregation.}
For each run, we average accuracy over evaluated OOD
lengths $\{20,25,30,\ldots,60\}$, then average these
per-run values across the available runs.
Oracle and policy results use their respective available
run sets.
The reported gap is the difference between these
aggregate means, rather than a paired estimate.

\begin{table}[t]
  \centering
  \small
  \caption{
  \textbf{Oracle and policy OOD accuracy across tasks.}
  Accuracies (\%) are averaged over evaluated lengths
  $\{20,25,\ldots,60\}$ and across available runs.
  Gap is oracle minus policy OOD accuracy,
  in percentage points.
  }
  \label{tab:oracle-policy-full}
  \begin{tabular}{llrrr}
    \toprule
    Task & Method & Oracle & Policy & Gap \\
    \midrule
    Addition
    & $K=20$     & 34.2 & 17.8 & 16.4 \\
    & $K=n$      & 22.1 & 14.3 & 7.8 \\
    & $K=n\pm5$  & 30.9 & 29.6 & 1.3 \\
    & RL-Halting & 45.0 & 39.2 & 5.8 \\
    \midrule
    Dyck-1
    & $K=20$     & 84.6 & 51.8 & 32.8 \\
    & $K=n$      & 66.6 & 34.2 & 32.4 \\
    & $K=n\pm5$  & 88.2 & 85.4 & 2.8 \\
    & RL-Halting & 97.5 & 93.2 & 4.3 \\
    \midrule
    Unique Set
    & $K=20$     & 66.1 & 24.4 & 41.7 \\
    & $K=n$      & 78.5 & 64.5 & 14.0 \\
    & $K=n\pm5$  & 73.2 & 68.8 & 4.4 \\
    & RL-Halting & 82.8 & 77.5 & 5.3 \\
    \midrule
    Copy
    & $K=20$     & 62.3 & 33.0 & 29.3 \\
    & $K=n$      & 68.7 & 43.8 & 24.9 \\
    & $K=n\pm5$  & 55.9 & 54.4 & 1.5 \\
    & RL-Halting & 43.2 & 41.3 & 1.9 \\
    \bottomrule
  \end{tabular}
\end{table}

\paragraph{Policy performance across tasks.}
RL-Halting achieves the highest policy OOD accuracy
on Addition, Dyck-1, and Unique Set
(\cref{tab:oracle-policy-full}).
Its advantage on these tasks therefore persists without
oracle depth selection.
On Copy, the randomised $K=n\pm5$ schedule instead
achieves the highest policy accuracy.

Both stochastic schedules have smaller aggregate
oracle--policy gaps than the deterministic schedules
on all four tasks.
The randomised $K=n\pm5$ schedule has the smallest gap
on every task, although it does not always achieve
the highest policy accuracy.

\paragraph{Interpreting the gap.}
At each evaluated length $n$, we define
\begin{align}
\mathrm{Gap}(n)
=
\mathrm{Acc}_{\mathrm{oracle}}(n)
-
\mathrm{Acc}_{\mathrm{policy}}(n).
\end{align}
A smaller gap indicates closer oracle and policy
accuracies under the reported evaluation protocols.
The oracle search horizon is $n+2$, whereas RL-Halting
uses a policy horizon of $30$.
Consequently, at lengths $20$ and $25$, the oracle
search does not necessarily include the policy-selected
depth.
Together with the use of separate available run sets,
this means that the reported gap should be interpreted
as an aggregate performance difference, rather than
a pure measure of depth-selection error.

\paragraph{Length-wise comparison on Unique Set.}
\cref{fig:oracle-policy-gap-unique} presents policy
accuracy and the oracle--policy gap as functions of
input length for Unique Set.

\begin{figure}[t]
  \centering
  \includegraphics[width=\textwidth]
    {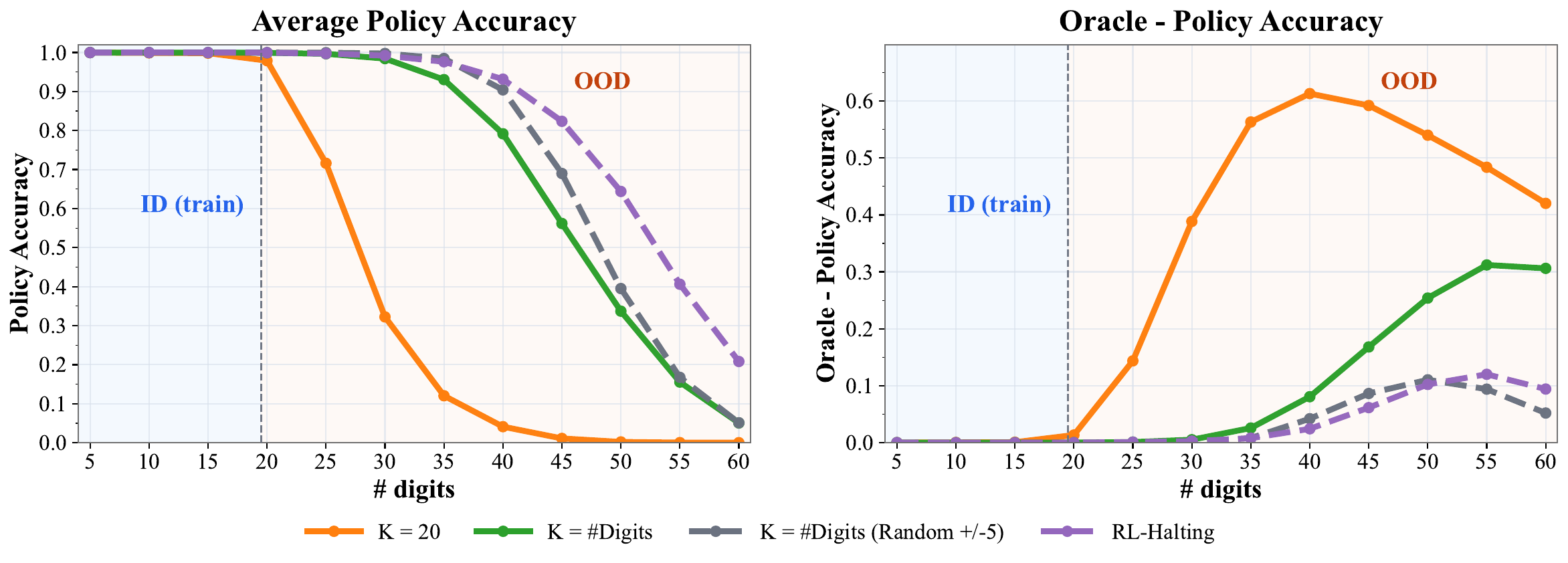}
  \caption{
  \textbf{Policy accuracy and oracle--policy gap on Unique Set.}
  Left: accuracy under each method's prescribed stopping rule.
  Right: oracle minus policy accuracy at each input length.
  Stopping rules are defined in \cref{sec:metrics}.
  }
  \label{fig:oracle-policy-gap-unique}
\end{figure}

The aggregate gap reductions complement the direct
depth-sweep and prediction-flip evidence in
\cref{fig:loop-trajectory-stability}.
Together, these results are consistent with stochastic
training reducing sensitivity to the inference-time
stopping depth.

A small oracle--policy gap does not by itself imply
strong extrapolation: both accuracies may be low.
For example, RL-Halting on Copy has a gap of only
$1.9$ percentage points but policy accuracy of $41.3\%$.
Reporting both accuracies therefore distinguishes
agreement between evaluation rules from the quality
of the selected predictions.

\subsection{Learned stopping schedule on binary addition}
\label{app:learned-stopping-schedule-addition}

We further inspect the stopping distribution learned by RL-Halting on binary addition. 
The goal is to understand whether the learned schedule collapses to an almost deterministic stopping rule, or whether it remains genuinely stochastic during length extrapolation.

For this diagnostic, we evaluate the learned stopping head over a visualisation horizon $T_{\mathrm{vis}}=60$. 
This horizon is used only to inspect the learned distribution and compare it with length-matched uniform schedules. 
For any horizon $T$, the same hazard rule defines a truncated stopping distribution $\pi_{\mathrm{stop},\phi}^{(T)}(\tau=k\mid x)$, with the remaining tail probability assigned to the final depth $T$.

We quantify the randomness of the learned schedule using the conditional entropy
\begin{align}
H_T(\tau\mid x)
=
-
\sum_{k=1}^{T}
\pi_{\mathrm{stop},\phi}^{(T)}(\tau=k\mid x)
\log_2
\pi_{\mathrm{stop},\phi}^{(T)}(\tau=k\mid x).
\end{align}

In this appendix, we set $T=T_{\mathrm{vis}}=60$ and report the entropy averaged over binary-addition examples of the same digit length,
\begin{align}
H(n)
=
\mathbb{E}_{x:\#\mathrm{digits}(x)=n}
\left[
H_{T_{\mathrm{vis}}}(\tau\mid x)
\right].
\end{align}

This measures randomness after conditioning on the input, and therefore distinguishes genuinely stochastic stopping from deterministic input-dependent schedules such as $K=n$, which have zero conditional entropy.

\begin{figure}[t]
  \centering
  \includegraphics[width=\textwidth]{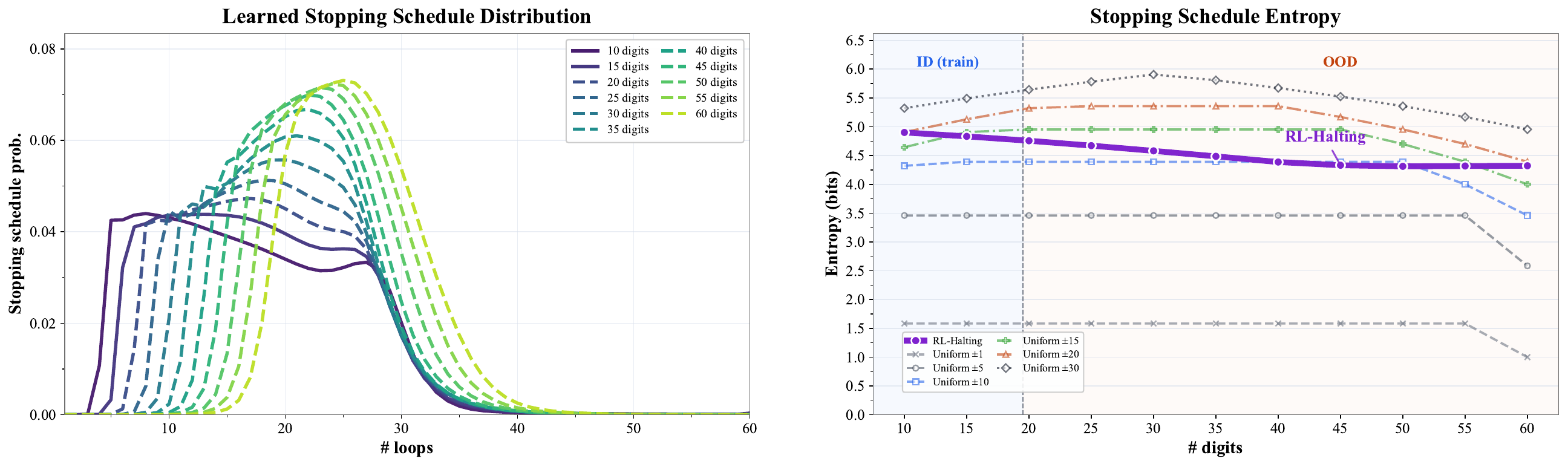}
  \caption{
  \textbf{Learned stopping schedule and stopping entropy on binary addition.}
  For this visualisation, the stopping distribution is truncated at $T_{\mathrm{vis}}=60$.
  \textbf{Left}: learned stopping distributions $\pi_{\mathrm{stop},\phi}^{(T_{\mathrm{vis}})}(\tau=k\mid x)$ for RL-Halting, averaged over binary-addition examples with the same digit length. 
  As the number of digits increases, the probability mass shifts toward larger loop depths, indicating that the learned policy adapts its stopping distribution to problem length. 
  \textbf{Right}: conditional stopping entropy as a function of digit length, compared with uniform local-window schedules. 
  The learned stopping distribution remains non-degenerate across ID and OOD digit lengths, with lower entropy than wide uniform schedules. This diagnostic uses $T_{\mathrm{vis}}=60$, whereas policy accuracy uses maximum-probability stopping with $T_{\mathrm{test}}=30$.
  % RL-Halting remains stochastic across both ID and OOD digit lengths, but has lower entropy than wide uniform schedules, suggesting that it learns a structured stochastic schedule rather than a maximally diffuse one.
  }
  \label{fig:learned-stopping-schedule-entropy-addition}
\end{figure}

\cref{fig:learned-stopping-schedule-entropy-addition} shows that RL-Halting learns a structured length-dependent distribution on binary addition. 
For short inputs, the distribution places substantial mass on smaller loop counts. 
As the number of digits increases, the mass shifts smoothly to the right, assigning higher probability to larger loop counts. 
Thus, the learned stopping rule is not simply a fixed-depth policy: it adapts the range of likely stopping times to the input length.

At the same time, the learned schedule does not collapse to a deterministic stopping depth. 
The entropy curve remains above zero across both ID
and OOD lengths, showing that the evaluated stopping
distribution remains non-degenerate.
However, its entropy is lower than that of wide uniform schedules around the input length. 
This suggests that RL-Halting learns a middle ground: it preserves stochasticity over loop depths, but concentrates probability on a length-dependent range of useful depths.

This behaviour is consistent with the role of RL-Halting in the binary-addition experiments. 
Unlike deterministic schedules, it does not force every update to use a single prescribed depth. 
Unlike hand-designed random schedules, it does not spread probability uniformly over a fixed window. 
Instead, it learns an adaptive stochastic schedule whose support shifts with digit length while maintaining moderate entropy. 
This provides a possible explanation for why RL-Halting reduces run-to-run OOD variability on addition while improving the accuracy--stability trade-off.
\subsection{Comparison with PonderNet-style halting}
\label{app:pondernet-comparison}

We also include a PonderNet-style learned halting baseline in \cref{tab:fixed-k-random-ablation}. 
PonderNet and RL-Halting use a similar hazard-based parameterisation of the stopping distribution, but differ in how this distribution is trained. 
A PonderNet-style objective optimises a distribution-weighted loss over loop depths, so depths with larger stopping probability also receive larger effective training weight. 
By contrast, RL-Halting samples a single stopping depth for each example, supervises the model only at that realised depth, and updates the stopping head with a policy-gradient estimator using the negative task loss as reward.

In our binary-addition experiments, the PonderNet-style baseline is substantially less stable than RL-Halting. 
As shown in \cref{tab:fixed-k-random-ablation}, PonderNet obtains lower OOD accuracy and a higher standard deviation across runs, while RL-Halting achieves both the best OOD accuracy and the lowest variance. 
Empirically, we find that the PonderNet-style objective is sensitive to data order and often collapses toward a nearly deterministic halting distribution, even with entropy regularisation coefficient $0.01$. 

This is consistent with the broader observation that learned halting or gating mechanisms require explicit regularisation to avoid degenerate stopping behaviour: PonderNet uses a KL prior partly to assign non-zero probability to all halting steps and promote exploration, while recent learned-depth models also report collapse or instability in learned gates \citep{banino_pondernet_2021,zhu_scaling_2025}.

One possible explanation is a self-reinforcing feedback loop: early batches can assign slightly higher stopping probability to some depths, those depths then receive larger effective loss weight, and the stopping distribution further concentrates around them.
% RL-Halting avoids this dense distribution-weighted supervision and instead preserves the sampled-depth training protocol, which may explain why it behaves more like a stable adaptive stochastic schedule in this setting.

\subsection{Sensitivity to the maximum training loop depth}
\label{app:depth-sensitivity}

We study how the maximum training loop depth affects RL-Halting
on binary addition by varying $T_{\max}\in\{5,10,15,20\}$.
Here, $T_{\max}$ denotes the training horizon
$T_{\mathrm{train}}$ in \cref{app:rl-halting}: it determines
the available stopping depths during training, with the
remaining stopping probability assigned to the final iteration.
This ablation concerns the training budget, rather than the
choice of stopping depth at inference.
We report oracle OOD accuracy as defined in \cref{sec:metrics}.

% Before submission, specify the fixed entropy coefficient,
% training steps, evaluation depth range, and seed count.
% Include uncertainty across runs if available.

\begin{table}[t]
  \centering
  \caption{
    \textbf{Sensitivity to the maximum training loop depth
    on binary addition.}
    Oracle OOD accuracy for RL-Halting trained with different
    maximum loop budgets.
  }
  \label{tab:depth-sensitivity}
  \begin{tabular}{cc}
    \toprule
    $T_{\max}$ & Oracle OOD accuracy (\%) $\uparrow$ \\
    \midrule
    5  & 26.26 \\
    10 & 48.70 \\
    15 & 45.81 \\
    20 & 45.00 \\
    \bottomrule
  \end{tabular}
\end{table}

As shown in \cref{tab:depth-sensitivity}, increasing the
training budget from $5$ to $10$ substantially improves
oracle OOD accuracy, from $26.26\%$ to $48.70\%$.
Increasing the budget further does not improve the reported
accuracy: the results for $T_{\max}\in\{10,15,20\}$ span
$3.70$ percentage points, with the highest value at
$T_{\max}=10$.

These results indicate that a larger training horizon is
not uniformly beneficial for extrapolation.
Changing the horizon changes both the available recurrent
computation and the support of the learned training stopping
distribution.
The lower accuracy at $T_{\max}=5$ is consistent with an
overly restrictive training budget, whereas the results at
larger budgets suggest that additional available depth alone
does not ensure a better extrapolating computation.

\subsection{Sensitivity to entropy regularisation}
\label{app:entropy-sensitivity}

We vary the entropy coefficient
$\lambda_{\mathrm{ent}}\in\{0,0.001,0.1\}$
on binary addition to examine how exploration in the learned
stopping distribution affects extrapolation.

\begin{table}[t]
  \centering
  \caption{
    \textbf{Entropy regularisation in RL-Halting on binary addition.}
    An intermediate coefficient gives the highest oracle OOD
    accuracy among the tested values.
  }
  \label{tab:entropy-sensitivity}
  \begin{tabular}{cc}
    \toprule
    $\lambda_{\mathrm{ent}}$
      & Oracle OOD accuracy (\%) $\uparrow$ \\
    \midrule
    0     & 12.51 \\
    0.001 & 45.00 \\
    0.1   & 31.35 \\
    \bottomrule
  \end{tabular}
\end{table}

Without entropy regularisation, the learned stopping distribution
becomes nearly deterministic around a small loop depth, and
oracle OOD accuracy is only $12.51\%$
(\cref{tab:entropy-sensitivity}).
Setting $\lambda_{\mathrm{ent}}=0.001$ prevents this collapse
and raises accuracy to $45.00\%$, whereas stronger regularisation
($\lambda_{\mathrm{ent}}=0.1$) reduces it to $31.35\%$.

These results support maintaining stochasticity in the training
stopping schedule while showing that stronger entropy
regularisation is not uniformly beneficial.
They complement \cref{sec:loss_rl_halting}: exposing the model
to multiple stopping depths can help, but the resulting
extrapolation depends on how the stopping distribution
allocates supervision across those depths.

\section{Toy Analysis for Randomised Loop Schedules}
\label{app:randomized-schedule-theory}

This appendix formalises the gradient-averaging intuition used in \cref{sec:toy-gradient-averaging}. 
The result should be read as a local mechanism for stabilisation, not as a guarantee that randomisation always improves extrapolation.

Let $z=(x,y)$ be a training example and let $L_k(\theta;z)$ denote the loss after running the looped model for $k$ iterations. 
Assume $L_k$ is differentiable in $\theta$, and define the depth-$k$ gradient
$
G_k(\theta;z):=\nabla_\theta L_k(\theta;z).
$
For a randomised schedule $\pi$ supported on a finite set of loop depths $S\subseteq\{1,\dots,T\}$, define the schedule-averaged objective
\begin{align}
J_\pi(\theta;z)
:=
\mathbb{E}_{K\sim \pi}\left[L_K(\theta;z)\right]
=
\sum_{k\in S}\pi(k)L_k(\theta;z).
\end{align}

Its gradient is
\begin{align}
G_\pi(\theta;z)
:=
\nabla_\theta J_\pi(\theta;z)
=
\sum_{k\in S}\pi(k)G_k(\theta;z).
\end{align}

Equivalently, if one samples $K\sim\pi$ at each update and uses the single-depth gradient $G_K(\theta;z)$, then
\begin{align}
\mathbb{E}_{K\sim\pi}\left[G_K(\theta;z)\right]
=
G_\pi(\theta;z).
\end{align}

Thus, sampling one loop count per update gives an unbiased estimator of the gradient of the schedule-averaged objective.

\begin{assumption}[Local shared-gradient model]
\label{assump:local-shared-gradient}
Fix $\theta$ and $z=(x,y)$. 
For loop depths $k\in S$, assume that the gradients admit the decomposition
\begin{align}
G_k(\theta;z)
=
\bar G(\theta;z)+\xi_k(\theta;z),
\end{align}

where $\bar G(\theta;z)$ is a shared gradient component and $\xi_k(\theta;z)$ is a depth-specific residual. 
Assume there exist constants $B>0$ and $\rho\in[0,1]$ such that
\begin{align}
\|\xi_k(\theta;z)\|^2\leq B^2
\qquad \text{for all } k\in S,
\end{align}

and
\begin{align}
\langle \xi_i(\theta;z),\xi_j(\theta;z)\rangle
\leq
\rho B^2
\qquad \text{for all } i\neq j,\; i,j\in S.
\end{align}

\end{assumption}

\begin{proposition}[Randomised schedules average depth-specific gradients]
\label{prop:gradient-averaging}
Under Assumption~\ref{assump:local-shared-gradient}, the schedule-averaged gradient satisfies
\begin{align}
\left\|
G_\pi(\theta;z)-\bar G(\theta;z)
\right\|^2
\leq
B^2
\left[
\rho+(1-\rho)\sum_{k\in S}\pi(k)^2
\right].
\end{align}

In particular, if $\pi$ is uniform on $S$ and $m=|S|$, then
\begin{align}
\left\|
G_\pi(\theta;z)-\bar G(\theta;z)
\right\|^2
\leq
B^2
\left(
\rho+\frac{1-\rho}{m}
\right).
\end{align}

Thus, when the depth-specific residuals are weakly aligned across loop depths, increasing the number of loop depths in the randomised window reduces the remaining depth-specific component of the expected gradient.
\end{proposition}

\begin{proof}
Using the decomposition $G_k=\bar G+\xi_k$ and $\sum_{k\in S}\pi(k)=1$, we have
\begin{align}
G_\pi-\bar G
=
\sum_{k\in S}\pi(k)G_k-\bar G
=
\sum_{k\in S}\pi(k)(\bar G+\xi_k)-\bar G
=
\sum_{k\in S}\pi(k)\xi_k.
\end{align}

Therefore,
\begin{align}
\left\|G_\pi-\bar G\right\|^2
=
\left\|
\sum_{k\in S}\pi(k)\xi_k
\right\|^2.
\end{align}

Expanding the squared norm gives
\begin{align}
\left\|G_\pi-\bar G\right\|^2
=
\sum_{k\in S}\pi(k)^2\|\xi_k\|^2
+
\sum_{\substack{i,j\in S\\ i\neq j}}
\pi(i)\pi(j)\langle \xi_i,\xi_j\rangle.
\end{align}

By Assumption~\ref{assump:local-shared-gradient},
\begin{align}
\left\|G_\pi-\bar G\right\|^2
\leq
B^2\sum_{k\in S}\pi(k)^2
+
\rho B^2
\sum_{\substack{i,j\in S\\ i\neq j}}
\pi(i)\pi(j).
\end{align}

Since
\begin{align}
\sum_{\substack{i,j\in S\\ i\neq j}}
\pi(i)\pi(j)
=
\left(\sum_{k\in S}\pi(k)\right)^2
-
\sum_{k\in S}\pi(k)^2
=
1-\sum_{k\in S}\pi(k)^2,
\end{align}
we obtain
\begin{align}
\left\|G_\pi-\bar G\right\|^2
\leq
B^2\sum_{k\in S}\pi(k)^2
+
\rho B^2
\left(
1-\sum_{k\in S}\pi(k)^2
\right).
\end{align}

Rearranging gives
\begin{align}
\left\|G_\pi-\bar G\right\|^2
\leq
B^2
\left[
\rho+(1-\rho)\sum_{k\in S}\pi(k)^2
\right].
\end{align}

If $\pi$ is uniform on $S$ and $m=|S|$, then $\pi(k)=1/m$ for all $k\in S$, so
\begin{align}
\sum_{k\in S}\pi(k)^2
=
m\cdot \frac{1}{m^2}
=
\frac{1}{m}.
\end{align}

Substituting this into the bound gives
\begin{align}
\left\|
G_\pi-\bar G
\right\|^2
\leq
B^2
\left(
\rho+\frac{1-\rho}{m}
\right).
\end{align}

\end{proof}

\paragraph{Interpretation.}
For the schedules used in \cref{sec:stochastic-schedules}, $S$ corresponds to a window of nearby loop depths around the input length, e.g.
$
K\sim \#\text{digits}+\mathrm{Unif}(-w,w).
$
The proposition shows that if nearby depths share a common useful gradient component and differ mainly through weakly aligned depth-specific residuals, then averaging over a larger window suppresses the depth-specific part of the expected gradient. 
This provides a possible mechanism for reduced OOD variance across runs.

The assumption is intentionally local. 
It is not expected to hold for arbitrary loop depths or arbitrarily large windows. 
Moreover, the proposition only explains stabilisation of the expected gradient around $\bar G$; it does not imply that $\bar G$ is an extrapolating direction. 
Therefore, lower variance under stochastic schedules need not imply higher mean OOD accuracy.
\section{Toy Analysis for Loop-Consistent Representations}
\label{app:loop-consistency-theory}

This appendix formalises the intuition that randomised schedules can encourage hidden states at different loop depths to become consistent. 
The result is a sufficient-condition argument: it does not claim that low loss always identifies a unique hidden representation, but shows that if successful predictions at different depths require closeness to a shared representation, then training across multiple depths forces those states to be close to one another.

Let $H_k^\theta(x)\in\mathbb{R}^{n\times d}$ be the full hidden state after $k$ loop iterations, and let
$
L_k(\theta;x,y)=\ell(\mathrm{Readout}(H_k^\theta(x)),y)
$
be the prediction loss at depth $k$. 
Let $S\subseteq\{1,\dots,T\}$ be a finite set of loop depths, and consider the randomised-depth objective
\begin{align}
J_S(\theta)
=
\mathbb{E}_{(x,y)\sim\mathcal D}
\left[
\frac{1}{|S|}
\sum_{k\in S} L_k(\theta;x,y)
\right].
\end{align}

\begin{assumption}[Shared target representation]
\label{assump:shared-target-representation}
For every input $x$, there exists a representation $Z^\star(x)\in\mathbb{R}^{n\times d}$ and a constant $\mu>0$ such that, for every $k\in S$,
\begin{align}
L_k(\theta;x,y)
\geq
\frac{\mu}{2}
\|H_k^\theta(x)-Z^\star(x)\|_F^2 .
\end{align}
\end{assumption}

\begin{proposition}[Randomised schedules encourage loop-consistent states]
\label{prop:loop-consistency}
Under Assumption~\ref{assump:shared-target-representation}, if
$
J_S(\theta)\leq \varepsilon,
$
then
\begin{align}
\mathbb{E}_{(x,y)\sim\mathcal D}
\left[
\frac{1}{|S|^2}
\sum_{i,j\in S}
\|H_i^\theta(x)-H_j^\theta(x)\|_F^2
\right]
\leq
\frac{8\varepsilon}{\mu}.
\end{align}
\end{proposition}

\begin{proof}
By Assumption~\ref{assump:shared-target-representation},
\begin{align}
J_S(\theta)
=
\mathbb{E}_{(x,y)}
\left[
\frac{1}{|S|}
\sum_{k\in S}L_k(\theta;x,y)
\right]
\geq
\frac{\mu}{2}
\mathbb{E}_{(x,y)}
\left[
\frac{1}{|S|}
\sum_{k\in S}
\|H_k^\theta(x)-Z^\star(x)\|_F^2
\right].
\end{align}

Therefore, if $J_S(\theta)\leq \varepsilon$, then
\begin{align}
\mathbb{E}_{(x,y)}
\left[
\frac{1}{|S|}
\sum_{k\in S}
\|H_k^\theta(x)-Z^\star(x)\|_F^2
\right]
\leq
\frac{2\varepsilon}{\mu}.
\end{align}

For any $i,j\in S$,
\begin{align}
\|H_i^\theta(x)-H_j^\theta(x)\|_F^2
=
\|(H_i^\theta(x)-Z^\star(x))-(H_j^\theta(x)-Z^\star(x))\|_F^2
\end{align}

and therefore
\begin{align}
\|H_i^\theta(x)-H_j^\theta(x)\|_F^2
\leq
2\|H_i^\theta(x)-Z^\star(x)\|_F^2
+
2\|H_j^\theta(x)-Z^\star(x)\|_F^2 .
\end{align}

Averaging over all pairs $(i,j)\in S\times S$ gives
\begin{align}
\frac{1}{|S|^2}
\sum_{i,j\in S}
\|H_i^\theta(x)-H_j^\theta(x)\|_F^2
\leq
\frac{4}{|S|}
\sum_{k\in S}
\|H_k^\theta(x)-Z^\star(x)\|_F^2 .
\end{align}

Taking expectation over $(x,y)\sim\mathcal D$ and using the previous bound yields
\begin{align}
\mathbb{E}_{(x,y)}
\left[
\frac{1}{|S|^2}
\sum_{i,j\in S}
\|H_i^\theta(x)-H_j^\theta(x)\|_F^2
\right]
\leq
4\cdot \frac{2\varepsilon}{\mu}
=
\frac{8\varepsilon}{\mu}.
\end{align}

\end{proof}

\paragraph{Interpretation.}
The proposition shows that, under a shared-representation assumption, low loss across a stochastic window implies consistency of hidden states within that window. 
This provides a toy explanation for why stochastic schedules can reduce prediction changes across adjacent loop counts. 
The assumption is intentionally strong: cross-entropy loss does not generally identify a unique hidden representation. 
The result should therefore be interpreted as a sufficient mechanism, not as a universal guarantee.

%%%%%%%%%%%%%%%%%%%%%%%%%%%%%%%%%%%%%%%%%%%%%%%%%%%%%%%%%%%%

\newpage
\input{checklist.tex}

\end{document}

%% file: checklist.tex
\section*{NeurIPS Paper Checklist}

\begin{enumerate}

\item {\bf Claims}
    \item[] Question: Do the main claims made in the abstract and introduction accurately reflect the paper's contributions and scope?
    \item[] Answer: \answerYes{} % Replace by \answerYes{}, \answerNo{}, or \answerNA{}.
    \item[] Justification: The abstract and introduction state the main contributions: identifying the extrapolation--stability gap in Looped Transformers, showing stabilisation via stochastic schedules, and studying RL-Halting as a learned stochastic stopping rule. These claims are supported by the experiments in \cref{sec:looping-enable,sec:stochastic-schedules,sec:loss_rl_halting} and are stated with the caveat that stability does not always imply extrapolation.
    \item[] Guidelines:
    \begin{itemize}
        \item The answer \answerNA{} means that the abstract and introduction do not include the claims made in the paper.
        \item The abstract and/or introduction should clearly state the claims made, including the contributions made in the paper and important assumptions and limitations. A \answerNo{} or \answerNA{} answer to this question will not be perceived well by the reviewers. 
        \item The claims made should match theoretical and experimental results, and reflect how much the results can be expected to generalize to other settings. 
        \item It is fine to include aspirational goals as motivation as long as it is clear that these goals are not attained by the paper. 
    \end{itemize}

\item {\bf Limitations}
    \item[] Question: Does the paper discuss the limitations of the work performed by the authors?
    \item[] Answer: \answerYes{} % Replace by \answerYes{}, \answerNo{}, or \answerNA{}.
    \item[] Justification: The paper discusses several limitations: stochastic schedules can stabilise suboptimal computations, RL-Halting underperforms on Copy, and the toy analyses explain stabilisation rather than guaranteeing extrapolation. These limitations are discussed in \cref{sec:stochastic-schedules,sec:loss_rl_halting,sec:toy-analysis,app:randomized-schedule-theory,app:loop-consistency-theory}.

    \item[] Guidelines:
    \begin{itemize}
        \item The answer \answerNA{} means that the paper has no limitation while the answer \answerNo{} means that the paper has limitations, but those are not discussed in the paper. 
        \item The authors are encouraged to create a separate ``Limitations'' section in their paper.
        \item The paper should point out any strong assumptions and how robust the results are to violations of these assumptions (e.g., independence assumptions, noiseless settings, model well-specification, asymptotic approximations only holding locally). The authors should reflect on how these assumptions might be violated in practice and what the implications would be.
        \item The authors should reflect on the scope of the claims made, e.g., if the approach was only tested on a few datasets or with a few runs. In general, empirical results often depend on implicit assumptions, which should be articulated.
        \item The authors should reflect on the factors that influence the performance of the approach. For example, a facial recognition algorithm may perform poorly when image resolution is low or images are taken in low lighting. Or a speech-to-text system might not be used reliably to provide closed captions for online lectures because it fails to handle technical jargon.
        \item The authors should discuss the computational efficiency of the proposed algorithms and how they scale with dataset size.
        \item If applicable, the authors should discuss possible limitations of their approach to address problems of privacy and fairness.
        \item While the authors might fear that complete honesty about limitations might be used by reviewers as grounds for rejection, a worse outcome might be that reviewers discover limitations that aren't acknowledged in the paper. The authors should use their best judgment and recognize that individual actions in favor of transparency play an important role in developing norms that preserve the integrity of the community. Reviewers will be specifically instructed to not penalize honesty concerning limitations.
    \end{itemize}

\item {\bf Theory assumptions and proofs}
    \item[] Question: For each theoretical result, does the paper provide the full set of assumptions and a complete (and correct) proof?
    \item[] Answer: \answerYes{} % Replace by \answerYes{}, \answerNo{}, or \answerNA{}.
    \item[] Justification: The theoretical statements are presented as toy mechanisms with explicit assumptions and complete proofs. The assumptions, propositions, proofs, and interpretation are given in \cref{app:randomized-schedule-theory,app:loop-consistency-theory}.
    \item[] Guidelines:
    \begin{itemize}
        \item The answer \answerNA{} means that the paper does not include theoretical results. 
        \item All the theorems, formulas, and proofs in the paper should be numbered and cross-referenced.
        \item All assumptions should be clearly stated or referenced in the statement of any theorems.
        \item The proofs can either appear in the main paper or the supplemental material, but if they appear in the supplemental material, the authors are encouraged to provide a short proof sketch to provide intuition. 
        \item Inversely, any informal proof provided in the core of the paper should be complemented by formal proofs provided in appendix or supplemental material.
        \item Theorems and Lemmas that the proof relies upon should be properly referenced. 
    \end{itemize}

    \item {\bf Experimental result reproducibility}
    \item[] Question: Does the paper fully disclose all the information needed to reproduce the main experimental results of the paper to the extent that it affects the main claims and/or conclusions of the paper (regardless of whether the code and data are provided or not)?
    \item[] Answer: \answerYes{} % Replace by \answerYes{}, \answerNo{}, or \answerNA{}.
    \item[] Justification: The paper describes the tasks, train--test length split, model families, stopping schedules, metrics, and implementation details needed to reproduce the main results. These details are provided in \cref{sec:setup,sec:tasks-protocol,sec:models-schedules,sec:metrics,app:implementation-details,app:rl-halting}.

    \item[] Guidelines:
    \begin{itemize}
        \item The answer \answerNA{} means that the paper does not include experiments.
        \item If the paper includes experiments, a \answerNo{} answer to this question will not be perceived well by the reviewers: Making the paper reproducible is important, regardless of whether the code and data are provided or not.
        \item If the contribution is a dataset and\slash or model, the authors should describe the steps taken to make their results reproducible or verifiable. 
        \item Depending on the contribution, reproducibility can be accomplished in various ways. For example, if the contribution is a novel architecture, describing the architecture fully might suffice, or if the contribution is a specific model and empirical evaluation, it may be necessary to either make it possible for others to replicate the model with the same dataset, or provide access to the model. In general. releasing code and data is often one good way to accomplish this, but reproducibility can also be provided via detailed instructions for how to replicate the results, access to a hosted model (e.g., in the case of a large language model), releasing of a model checkpoint, or other means that are appropriate to the research performed.
        \item While NeurIPS does not require releasing code, the conference does require all submissions to provide some reasonable avenue for reproducibility, which may depend on the nature of the contribution. For example
        \begin{enumerate}
            \item If the contribution is primarily a new algorithm, the paper should make it clear how to reproduce that algorithm.
            \item If the contribution is primarily a new model architecture, the paper should describe the architecture clearly and fully.
            \item If the contribution is a new model (e.g., a large language model), then there should either be a way to access this model for reproducing the results or a way to reproduce the model (e.g., with an open-source dataset or instructions for how to construct the dataset).
            \item We recognize that reproducibility may be tricky in some cases, in which case authors are welcome to describe the particular way they provide for reproducibility. In the case of closed-source models, it may be that access to the model is limited in some way (e.g., to registered users), but it should be possible for other researchers to have some path to reproducing or verifying the results.
        \end{enumerate}
    \end{itemize}

\item {\bf Open access to data and code}
    \item[] Question: Does the paper provide open access to the data and code, with sufficient instructions to faithfully reproduce the main experimental results, as described in supplemental material?
    \item[] Answer: \answerYes{} % Replace by \answerYes{}, \answerNo{}, or \answerNA{}.
    \item[] Justification: We provide implementations of the studied stopping schedules and
selected experimental configurations at
\url{https://github.com/Hsun-Yu/looped_stopping}.

    \item[] Guidelines:
    \begin{itemize}
        \item The answer \answerNA{} means that paper does not include experiments requiring code.
        \item Please see the NeurIPS code and data submission guidelines (\url{https://neurips.cc/public/guides/CodeSubmissionPolicy}) for more details.
        \item While we encourage the release of code and data, we understand that this might not be possible, so \answerNo{} is an acceptable answer. Papers cannot be rejected simply for not including code, unless this is central to the contribution (e.g., for a new open-source benchmark).
        \item The instructions should contain the exact command and environment needed to run to reproduce the results. See the NeurIPS code and data submission guidelines (\url{https://neurips.cc/public/guides/CodeSubmissionPolicy}) for more details.
        \item The authors should provide instructions on data access and preparation, including how to access the raw data, preprocessed data, intermediate data, and generated data, etc.
        \item The authors should provide scripts to reproduce all experimental results for the new proposed method and baselines. If only a subset of experiments are reproducible, they should state which ones are omitted from the script and why.
        \item At submission time, to preserve anonymity, the authors should release anonymized versions (if applicable).
        \item Providing as much information as possible in supplemental material (appended to the paper) is recommended, but including URLs to data and code is permitted.
    \end{itemize}

\item {\bf Experimental setting/details}
    \item[] Question: Does the paper specify all the training and test details (e.g., data splits, hyperparameters, how they were chosen, type of optimizer) necessary to understand the results?
    \item[] Answer: \answerYes{} % Replace by \answerYes{}, \answerNo{}, or \answerNA{}.
    \item[] Justification: The paper specifies the task definitions, ID/OOD length regimes, model baselines, stopping schedules, evaluation metrics, optimiser, batch size, learning rate schedule, training steps, and number of independent runs. These details are given in \cref{sec:tasks-protocol,sec:models-schedules,sec:metrics,app:implementation-details}.

    \item[] Guidelines:
    \begin{itemize}
        \item The answer \answerNA{} means that the paper does not include experiments.
        \item The experimental setting should be presented in the core of the paper to a level of detail that is necessary to appreciate the results and make sense of them.
        \item The full details can be provided either with the code, in appendix, or as supplemental material.
    \end{itemize}

\item {\bf Experiment statistical significance}
    \item[] Question: Does the paper report error bars suitably and correctly defined or other appropriate information about the statistical significance of the experiments?
    \item[] Answer: \answerYes{} % Replace by \answerYes{}, \answerNo{}, or \answerNA{}.
    \item[] Justification: The paper reports variability across independent runs using individual-run curves, standard deviations, min--max ranges, and controlled variance experiments. These results are shown in \cref{fig:accuracy-runs-comparison,fig:random-window-ood-summary,tab:ood_generalization,fig:ood-variance-comparison}.

    \item[] Guidelines:
    \begin{itemize}
        \item The answer \answerNA{} means that the paper does not include experiments.
        \item The authors should answer \answerYes{} if the results are accompanied by error bars, confidence intervals, or statistical significance tests, at least for the experiments that support the main claims of the paper.
        \item The factors of variability that the error bars are capturing should be clearly stated (for example, train/test split, initialization, random drawing of some parameter, or overall run with given experimental conditions).
        \item The method for calculating the error bars should be explained (closed form formula, call to a library function, bootstrap, etc.)
        \item The assumptions made should be given (e.g., Normally distributed errors).
        \item It should be clear whether the error bar is the standard deviation or the standard error of the mean.
        \item It is OK to report 1-sigma error bars, but one should state it. The authors should preferably report a 2-sigma error bar than state that they have a 96\% CI, if the hypothesis of Normality of errors is not verified.
        \item For asymmetric distributions, the authors should be careful not to show in tables or figures symmetric error bars that would yield results that are out of range (e.g., negative error rates).
        \item If error bars are reported in tables or plots, the authors should explain in the text how they were calculated and reference the corresponding figures or tables in the text.
    \end{itemize}

\item {\bf Experiments compute resources}
    \item[] Question: For each experiment, does the paper provide sufficient information on the computer resources (type of compute workers, memory, time of execution) needed to reproduce the experiments?
    \item[] Answer: \answerYes{} % Replace by \answerYes{}, \answerNo{}, or \answerNA{}.
    \item[] Justification: The paper reports the hardware used, approximate per-run wall-clock time, and number of independent seeds. These compute details are provided in \cref{app:implementation-details}.
    \item[] Guidelines:
    \begin{itemize}
        \item The answer \answerNA{} means that the paper does not include experiments.
        \item The paper should indicate the type of compute workers CPU or GPU, internal cluster, or cloud provider, including relevant memory and storage.
        \item The paper should provide the amount of compute required for each of the individual experimental runs as well as estimate the total compute. 
        \item The paper should disclose whether the full research project required more compute than the experiments reported in the paper (e.g., preliminary or failed experiments that didn't make it into the paper). 
    \end{itemize}
    
\item {\bf Code of ethics}
    \item[] Question: Does the research conducted in the paper conform, in every respect, with the NeurIPS Code of Ethics \url{https://neurips.cc/public/EthicsGuidelines}?
    \item[] Answer:  \answerYes{} % Replace by \answerYes{}, \answerNo{}, or \answerNA{}.
    \item[] Justification: The research uses controlled synthetic algorithmic sequence tasks and does not involve human subjects, private data, deception, harmful data collection, or deployment of high-risk systems. The task setup is described in \cref{sec:tasks-protocol}.

    \item[] Guidelines:
    \begin{itemize}
        \item The answer \answerNA{} means that the authors have not reviewed the NeurIPS Code of Ethics.
        \item If the authors answer \answerNo, they should explain the special circumstances that require a deviation from the Code of Ethics.
        \item The authors should make sure to preserve anonymity (e.g., if there is a special consideration due to laws or regulations in their jurisdiction).
    \end{itemize}

\item {\bf Broader impacts}
    \item[] Question: Does the paper discuss both potential positive societal impacts and negative societal impacts of the work performed?
    \item[] Answer: \answerNo{} % Replace by \answerYes{}, \answerNo{}, or \answerNA{}.
    \item[] Justification: 
    This paper studies the reliability of length generalisation in Looped
Transformers using controlled synthetic algorithmic tasks. It does not
evaluate deployment applications or their societal consequences. We
therefore limit our claims to the experimental setting studied and do not
provide a broader societal-impact analysis.
    % The paper does not include a dedicated broader-impact discussion. The work is primarily foundational and studies controlled synthetic sequence tasks, as described in \cref{sec:tasks-protocol}.

    \item[] Guidelines:
    \begin{itemize}
        \item The answer \answerNA{} means that there is no societal impact of the work performed.
        \item If the authors answer \answerNA{} or \answerNo, they should explain why their work has no societal impact or why the paper does not address societal impact.
        \item Examples of negative societal impacts include potential malicious or unintended uses (e.g., disinformation, generating fake profiles, surveillance), fairness considerations (e.g., deployment of technologies that could make decisions that unfairly impact specific groups), privacy considerations, and security considerations.
        \item The conference expects that many papers will be foundational research and not tied to particular applications, let alone deployments. However, if there is a direct path to any negative applications, the authors should point it out. For example, it is legitimate to point out that an improvement in the quality of generative models could be used to generate Deepfakes for disinformation. On the other hand, it is not needed to point out that a generic algorithm for optimizing neural networks could enable people to train models that generate Deepfakes faster.
        \item The authors should consider possible harms that could arise when the technology is being used as intended and functioning correctly, harms that could arise when the technology is being used as intended but gives incorrect results, and harms following from (intentional or unintentional) misuse of the technology.
        \item If there are negative societal impacts, the authors could also discuss possible mitigation strategies (e.g., gated release of models, providing defenses in addition to attacks, mechanisms for monitoring misuse, mechanisms to monitor how a system learns from feedback over time, improving the efficiency and accessibility of ML).
    \end{itemize}
    
\item {\bf Safeguards}
    \item[] Question: Does the paper describe safeguards that have been put in place for responsible release of data or models that have a high risk for misuse (e.g., pre-trained language models, image generators, or scraped datasets)?
    \item[] Answer: \answerNA{} % Replace by \answerYes{}, \answerNo{}, or \answerNA{}.
    \item[] Justification: The paper does not release high-risk models, pretrained generative models, scraped datasets, or systems intended for deployment. The experiments use controlled synthetic algorithmic tasks described in \cref{sec:tasks-protocol}.

    \item[] Guidelines:
    \begin{itemize}
        \item The answer \answerNA{} means that the paper poses no such risks.
        \item Released models that have a high risk for misuse or dual-use should be released with necessary safeguards to allow for controlled use of the model, for example by requiring that users adhere to usage guidelines or restrictions to access the model or implementing safety filters. 
        \item Datasets that have been scraped from the Internet could pose safety risks. The authors should describe how they avoided releasing unsafe images.
        \item We recognize that providing effective safeguards is challenging, and many papers do not require this, but we encourage authors to take this into account and make a best faith effort.
    \end{itemize}

\item {\bf Licenses for existing assets}
    \item[] Question: Are the creators or original owners of assets (e.g., code, data, models), used in the paper, properly credited and are the license and terms of use explicitly mentioned and properly respected?
    \item[] Answer: \answerNA{} % Replace by \answerYes{}, \answerNo{}, or \answerNA{}.
    \item[] Justification: The paper does not use existing datasets or pretrained models as experimental inputs. The experiments are conducted on synthetic sequence tasks generated according to the task definitions in \cref{sec:tasks-protocol}.

    \item[] Guidelines:
    \begin{itemize}
        \item The answer \answerNA{} means that the paper does not use existing assets.
        \item The authors should cite the original paper that produced the code package or dataset.
        \item The authors should state which version of the asset is used and, if possible, include a URL.
        \item The name of the license (e.g., CC-BY 4.0) should be included for each asset.
        \item For scraped data from a particular source (e.g., website), the copyright and terms of service of that source should be provided.
        \item If assets are released, the license, copyright information, and terms of use in the package should be provided. For popular datasets, \url{paperswithcode.com/datasets} has curated licenses for some datasets. Their licensing guide can help determine the license of a dataset.
        \item For existing datasets that are re-packaged, both the original license and the license of the derived asset (if it has changed) should be provided.
        \item If this information is not available online, the authors are encouraged to reach out to the asset's creators.
    \end{itemize}

\item {\bf New assets}
    \item[] Question: Are new assets introduced in the paper well documented and is the documentation provided alongside the assets?
    \item[] Answer: \answerNA{} % Replace by \answerYes{}, \answerNo{}, or \answerNA{}.
    \item[] Justification: The paper does not introduce or release a new dataset, benchmark asset, pretrained model. The synthetic task definitions and evaluation protocol are described in \cref{sec:tasks-protocol,sec:metrics}.

    \item[] Guidelines:
    \begin{itemize}
        \item The answer \answerNA{} means that the paper does not release new assets.
        \item Researchers should communicate the details of the dataset\slash code\slash model as part of their submissions via structured templates. This includes details about training, license, limitations, etc. 
        \item The paper should discuss whether and how consent was obtained from people whose asset is used.
        \item At submission time, remember to anonymize your assets (if applicable). You can either create an anonymized URL or include an anonymized zip file.
    \end{itemize}

\item {\bf Crowdsourcing and research with human subjects}
    \item[] Question: For crowdsourcing experiments and research with human subjects, does the paper include the full text of instructions given to participants and screenshots, if applicable, as well as details about compensation (if any)? 
    \item[] Answer: \answerNA{} % Replace by \answerYes{}, \answerNo{}, or \answerNA{}.
    \item[] Justification: The paper does not involve crowdsourcing, user studies, or research with human subjects. All experiments are performed on synthetic algorithmic sequence tasks, as described in \cref{sec:tasks-protocol}.

    \item[] Guidelines:
    \begin{itemize}
        \item The answer \answerNA{} means that the paper does not involve crowdsourcing nor research with human subjects.
        \item Including this information in the supplemental material is fine, but if the main contribution of the paper involves human subjects, then as much detail as possible should be included in the main paper. 
        \item According to the NeurIPS Code of Ethics, workers involved in data collection, curation, or other labor should be paid at least the minimum wage in the country of the data collector. 
    \end{itemize}

\item {\bf Institutional review board (IRB) approvals or equivalent for research with human subjects}
    \item[] Question: Does the paper describe potential risks incurred by study participants, whether such risks were disclosed to the subjects, and whether Institutional Review Board (IRB) approvals (or an equivalent approval/review based on the requirements of your country or institution) were obtained?
    \item[] Answer: \answerNA{} % Replace by \answerYes{}, \answerNo{}, or \answerNA{}.
    \item[] Justification: The paper does not involve human subjects, crowdsourcing, or participant data collection. Therefore, institutional review board approval or equivalent human-subjects review is not applicable.

    \item[] Guidelines:
    \begin{itemize}
        \item The answer \answerNA{} means that the paper does not involve crowdsourcing nor research with human subjects.
        \item Depending on the country in which research is conducted, IRB approval (or equivalent) may be required for any human subjects research. If you obtained IRB approval, you should clearly state this in the paper. 
        \item We recognize that the procedures for this may vary significantly between institutions and locations, and we expect authors to adhere to the NeurIPS Code of Ethics and the guidelines for their institution. 
        \item For initial submissions, do not include any information that would break anonymity (if applicable), such as the institution conducting the review.
    \end{itemize}

\item {\bf Declaration of LLM usage}
    \item[] Question: Does the paper describe the usage of LLMs if it is an important, original, or non-standard component of the core methods in this research? Note that if the LLM is used only for writing, editing, or formatting purposes and does \emph{not} impact the core methodology, scientific rigor, or originality of the research, declaration is not required.
    %this research? 
    \item[] Answer: \answerNA{} % Replace by \answerYes{}, \answerNo{}, or \answerNA{}.
    \item[] Justification: LLMs are not used as an important, original, or non-standard component of the method or experiments. The core methodology consists of Looped Transformers, stochastic stopping schedules, and RL-Halting, described in \cref{sec:models-schedules,sec:loss_rl_halting,app:rl-halting}.
    \item[] Guidelines:
    \begin{itemize}
        \item The answer \answerNA{} means that the core method development in this research does not involve LLMs as any important, original, or non-standard components.
        \item Please refer to our LLM policy in the NeurIPS handbook for what should or should not be described.
    \end{itemize}

\end{enumerate}